\documentclass[letterpaper]{article} 
\usepackage{aaai2026}  
\usepackage{times}  
\usepackage{helvet}  
\usepackage{courier}  
\usepackage[hyphens]{url}  
\usepackage{graphicx} 
\usepackage{natbib}  
\usepackage{caption} 
\usepackage{algorithm}
\usepackage{algorithmic}
\usepackage{tikz}
\usetikzlibrary{positioning,fit,decorations.pathmorphing}
\usetikzlibrary{arrows,shapes,positioning,shadows,trees,calc}
\usepackage{multirow}
\usepackage{amsfonts}
\usepackage{amsmath}
\usepackage{amsthm}
\usepackage{booktabs} 
\newtheorem{definition}{Definition}
\newtheorem{theorem}{Theorem}
\newtheorem{lemma}{Lemma}

\newtheorem{corollary}{Corollary}
\newtheorem{assumption}{Assumption}

\newtheorem{remark}{Remark}
\newtheorem{claim}{Claim}

\usepackage{newfloat}
\usepackage{listings}
\DeclareCaptionStyle{ruled}{labelfont=normalfont,labelsep=colon,strut=off} 
\floatstyle{ruled}
\newfloat{listing}{tb}{lst}{}
\floatname{listing}{Listing}
\title{Information Fidelity in Tool-Using LLM Agents:\\
A Martingale Analysis of the Model Context Protocol}

\author {
    Flint Xiaofeng Fan\textsuperscript{\rm 1,\rm 3},
    Cheston Tan\textsuperscript{\rm 1},
    Roger Wattenhofer\textsuperscript{\rm 3},
    Yew-Soon Ong\textsuperscript{\rm 1,\rm 2}
}
\affiliations {
    \textsuperscript{\rm 1}CFAR, IHPC, Agency for Science, Technology and Research, Singapore\\
    \textsuperscript{\rm 2}College of Computing and Data Science, Nanyang Technological University, Singapore\\
    \textsuperscript{\rm 3}Department of Information Technology and Electrical Engineering, ETH Zurich, Switzerland \\
    fxf@u.nus.edu, cheston-tan@a-star.edu.sg, wattenhofer@ethz.ch, asysong@ntu.edu.sg
}

\usepackage{bibentry}

\begin{document}

\maketitle

\begin{abstract}
As AI agents powered by large language models (LLMs) increasingly use external tools for high-stakes decisions, a critical reliability question arises: how do errors propagate across sequential tool calls? We introduce the first theoretical framework for analyzing error accumulation in Model Context Protocol (MCP) agents, proving that cumulative distortion exhibits linear growth and high-probability deviations bounded by $O(\sqrt{T})$. This concentration property ensures predictable system behavior and rules out exponential failure modes. We develop a hybrid distortion metric combining discrete fact matching with continuous semantic similarity, then establish martingale concentration bounds on error propagation through sequential tool interactions. Experiments across Qwen2-7B, Llama-3-8B, and Mistral-7B validate our theoretical predictions, showing empirical distortion tracks the linear trend with deviations consistently within $O(\sqrt{T})$ envelopes. Key findings include: semantic weighting reduces distortion by 80\%, and periodic re-grounding approximately every 9 steps suffices for error control. We translate these concentration guarantees into actionable deployment principles for trustworthy agent systems. The codebase is available at \url{https://github.com/flint-xf-fan/MCP}.
\end{abstract}


\section{Introduction}\label{sec:intro}
Large language models (LLMs) have achieved remarkable capabilities across natural language processing tasks, yet they remain fundamentally constrained by the static snapshots of knowledge encoded in their training corpora \cite{villalobos2022will,naveed2023comprehensive,li2024knowledge}. Once deployed, an LLM cannot ingest new facts or correct errors without costly retraining, leading to stale or even blatantly incorrect outputs in rapidly changing domains such as real-time news, financial markets, or clinical guidelines \cite{bommasani2021opportunities, liang2022holistic, fan2025fedrlhf}. 
As Silver and Sutton~\cite{silver2025era} aptly observe, ``we stand on the threshold of a new era in artificial intelligence'' where systems must transcend fixed datasets and learn through dynamic interaction.

To address this fundamental limitation, the Model Context Protocol (MCP) has emerged as one popular tool for connecting LLMs to external tools and data sources \cite{anthropic2024mcp}. 
MCP transforms what would traditionally be an $M\times N$ integration problem (with $M$ different AI applications needing custom connections to $N$ different tools) into a more manageable $M+N$ approach through a unified JSON-RPC framework, enabling seamless composition of models and services, as illustrated in Figure~\ref{fig:mcp_workflow}. 
Major platforms including Anthropic and Microsoft have implemented MCP-compatible interfaces \cite{anthropic2024mcp,microsoft2025mcpazure}, facilitating the development of systems that can retrieve up-to-date information, perform precise calculations, and interact with real-world services. Through MCP, LLMs can continuously integrate fresh external evidence to stay epistemically grounded and operationally up-to-date long after their initial training, instantiating Silver and Sutton's concept of experiential learning~\cite{silver2025era}.

\begin{figure}[t]
  \centering
    \includegraphics[width=0.7\linewidth]{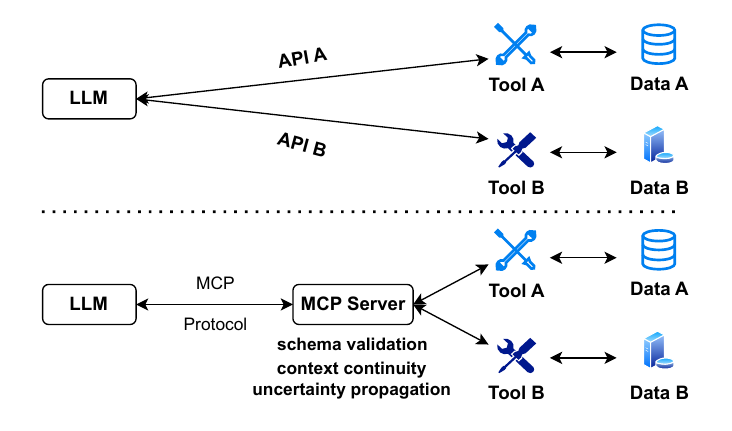}
  \caption{MCP standardizes LLM-tool integration through a unified JSON-RPC interface (bottom), replacing custom per-tool connections (top) with centralized schema validation, context management, and uncertainty propagation.}
  \label{fig:mcp_workflow}
\end{figure}

However, this enhanced agility puts \emph{information fidelity} at risk: each call to an external tool introduces opportunities for factual error, semantic drift, and compounding inaccuracies \cite{qin2025don}. In safety-critical applications, ranging from clinical decision support to financial analysis, even a minor mistake in an early query can cascade into catastrophic downstream consequences due to the sequential nature of decision-making processes~\cite{fan2021fault, fan2025position}. 
For instance, in 2012, Knight Capital lost \$440 million due to a software error in their trading algorithm, highlighting how small mistakes can lead to significant financial consequences \cite{dolfing2019}. Similarly, studies have shown that AI systems can provide inaccurate medical drug dosages, potentially leading to dangerous treatment recommendations \cite{ramasubramanian2024comparative}.

As tool usage becomes increasingly central to deployed AI systems, ensuring information fidelity across sequential interactions emerges as a critical challenge. Despite MCP's widespread adoption in production systems, we lack a formal framework for understanding how errors accumulate across adaptive queries, or what guarantees can be provided regarding the integrity of multi-step reasoning chains. While empirical work has demonstrated that function-calling improves factuality \cite{schick2023toolformer, yao2023react}, rigorous theoretical bounds on error propagation remain elusive.

In this paper, we bridge this gap by providing, to our knowledge, the first work to derive \emph{general high-probability deviation bounds} for cumulative hybrid semantic distortion in MCP-style, tool-augmented LLM pipelines under explicit dependence assumptions.
A martingale represents a mathematical ``fair game'' where future expectations equal the current state---intuitively preventing errors from spiraling uncontrollably across sequential interactions.
By adapting concentration techniques from adaptive data analysis \cite{howard2021selfnormalized, dwork2015preserving} and applying them to sequential tool interactions, we establish formal connections between everyday engineering practice and statistical learning theory. Our analysis models MCP exchanges as a bounded-difference martingale with finite dependency horizons, enabling us to derive high-probability bounds on cumulative semantic distortion. Our main contributions are:

\begin{itemize}
  \item \textbf{Information Fidelity Framework (Section~\ref{sec:framework}).}  We develop a comprehensive theoretical framework for analyzing MCP interactions, including a novel semantic distortion metric that combines weighted discrete fact matching with continuous embedding-based similarity, capturing both strict factual correctness and nuanced semantic drift. 
\item \textbf{Concentration Bounds (Section~\ref{sec:concentration}).} Within our framework, we derive high-probability concentration bounds showing cumulative semantic distortion deviations are bounded by $O(\sqrt{T \ln(1/\eta)})$, with experiments confirming consistent linear growth at constant per-step rates across sequential MCP interactions.
  
  \item \textbf{Empirical Validation (Section~\ref{sec:experiments}).} 
    Through experiments with Qwen2-7B-Instruct, Llama-3-8B-Instruct,
and Mistral-7B-Instruct-v0.3, we demonstrate that empirical distortion exhibits the predicted linear growth with $O(\sqrt{T})$ concentration bounds across varying dependency strengths.

\item \textbf{Practical Insights   (Section~\ref{sec:discussion}).} 
We derive actionable design principles from theoretical results, including optimal re-grounding intervals and distortion monitoring strategies, enabling safe deployment of MCP agents in safety-critical domains where unchecked error accumulation risks catastrophic failures.

\end{itemize}

\section{Preliminaries}
\label{sec:related}

\paragraph{Tool-Augmented Language Models}  
A rich line of work teaches LLMs to invoke external tools, e.g.\ Toolformer~\cite{schick2023toolformer} and ToolLLM~\cite{qin2023toolllm}. ReAct \cite{yao2023react} interleaves reasoning and tool calls. While these methods boost performance, none provide formal guarantees on how errors accumulate across multiple invocations.

\paragraph{Retrieval-Augmented Generation}  
RAG systems \cite{lewis2020retrieval} empower LLMs with external corpus lookups. Recent studies analyzing retrieval errors include RARR \cite{gao2023rarr} and Faithful reasoning \cite{creswell2022faithful}. Despite advances in index refinement \cite{ram2023context} and end-to-end training \cite{borgeaud2022improving,guu2020retrieval}, the theoretical reliability of repeated retrieval--generation loops remains open.

\paragraph{Model Context Protocol}  
General RPC/API standards, such as JSON-RPC \cite{jsonrpc2013} and OpenAPI \cite{openapi2021}, and the Language Server Protocol \cite{microsoft2022lsp} were not designed for LLM workflows. MCP \cite{anthropic2024mcp} fills this gap, prescribing schema validation, context continuity, and uncertainty handling. Its emergence is evidenced by platforms like Anthropic~\cite{anthropic2024mcp} and Microsoft Azure \cite{microsoft2025mcpazure}. Figure~\ref{fig:mcp_workflow} illustrates the typical flow: an LLM issues a JSON-RPC request validated by an MCP proxy, which routes it to appropriate tools and returns structured responses.
A primer on MCP is provided in Appendix~\ref{app:mcp_primer}.

\paragraph{Martingale-Based Analyses}  
Concentration bounds such as the Azuma--Hoeffding \cite{azuma1967,hoeffding1963probability} and Freedman's refinement \cite{freedman1975} underpin online learning and adaptive control \cite{dwork2015preserving}. A martingale is a stochastic process $\{X_t\}_{t=0}^T$ where $\mathbb{E}[X_{t+1} | X_0, X_1, \ldots, X_t] = X_t$, formalizing the concept of a ``fair game.`` Modern self-normalized variants \cite{howard2021selfnormalized} tighten these results, but to our knowledge no work has used them to quantify distortion in LLM--tool pipelines. The Azuma--Hoeffding inequality states that for a martingale with bounded differences $|X_{t+1} - X_t| \leq c_t$:
$\Pr\left( X_T - X_0 \geq \delta \right) \leq \exp\left( -\frac{\delta^2}{2\sum_{t=0}^{T-1} c_t^2} \right)$.

To our knowledge, no prior work unifies these strands into a principled, end-to-end theory of error accumulation in MCP-based LLM systems. 
In the next section, we develop our Information Fidelity Framework, which formalizes MCP interactions as an adaptive stochastic process and establishes the theoretical foundation for analyzing semantic distortion.

\section{Information Fidelity Framework}
\label{sec:framework}

To address the gap in prior work, we propose an Information Fidelity Framework that models MCP interactions as an adaptive stochastic process, quantifying error accumulation through a novel distortion metric and martingale analysis. Our framework formalizes how information propagates and potentially degrades through sequential MCP interactions, providing theoretical foundations for bounding semantic distortion across tool calls. We begin by establishing basic definitions and our semantic distortion metric, followed by key modeling assumptions, and finally derive the properties that enable our concentration results.


\begin{definition}[Information Filtration]
\label{def:filtration}
The sequence of pairs of query--response $\{(Q_t,R_t)\}$ is modeled via the natural filtration $\mathcal{F}_t=\sigma(Q_1,R_1,\dots,Q_t,R_t)$, representing all information available after $t$ interactions.
\end{definition}

\begin{definition}[Influence Function]
\label{def:influence}
The parameter $\beta\in(0,1)$ controls exponential decay of past influences through the influence function $\phi(i,j) = \beta^{j-i}$ for $1 \leq i < j \leq T$.

\end{definition}

\begin{definition}[Ideal Fact Set]
\label{def:ideal_facts}
For each prompt--response pair \((Q_t,R_t)\), let \(\mathcal{I}_t\) be the set of ground-truth facts (e.g.\ triples or attribute--value pairs) that \(R_t\) is expected to include.
\end{definition}

\subsection{Quantifying Semantic Distortion}
\label{subsec:distortion}

Measuring information fidelity requires balancing strict factual correctness against broader semantic similarity. We introduce a hybrid metric that captures both dimensions:
\begin{equation}
\label{eq:d_sem}
  \Delta_t = d_{\mathrm{sem}}\bigl(R_t,\mathcal{I}_t\bigr)
  =
  (1-\lambda)\,d^{w}_{\mathrm{set}}(R_t,\mathcal{I}_t)
  +
  \lambda\,d_{\mathrm{emb}}(R_t,\mathcal{I}_t),
\end{equation}
where:
\begin{align*}
  d^w_{\mathrm{set}}(R_t,\mathcal{I}_t)
  &= 1 
     - \frac{\sum_{f\in\hat{\mathcal{I}}_t\cap\mathcal{I}_t} w(f)}
            {\sum_{f\in\hat{\mathcal{I}}_t\cup\mathcal{I}_t} w(f)}, 
  \quad \hat{\mathcal{I}}_t=\text{extract}(R_t),
  \\
  d_{\mathrm{emb}}(R_t,\mathcal{I}_t)
  &= \tfrac{1 - \cos\bigl(\mathrm{embed}(R_t),\,\mathrm{embed}(\mathcal{I}_t)\bigr)}{2}.
\end{align*}

Here, $\text{extract}(R_t)$ is a function that extracts the set of facts from response $R_t$, $w(f)$ assigns importance weights to individual facts, and $\mathrm{embed}(\cdot)$ maps text to a semantic vector space. We define $\mathrm{embed}(\mathcal{I}_t) = \frac{1}{|\mathcal{I}_t|} \sum_{f \in \mathcal{I}_t} \mathrm{embed}(f)$, or for weighted facts, the weighted average of their embeddings.

The first component, \(d^w_{\mathrm{set}}\), measures factual completeness through weighted fact matching, penalizing both missing and incorrect facts. The second component, \(d_{\mathrm{emb}}\), captures semantic similarity in embedding space, addressing nuanced meaning beyond exact fact matching. The parameter $\lambda \in [0,1]$ allows tuning the trade-off between factual precision and semantic similarity. The choice of $\lambda$ depends on the application; for example, $\lambda = 0.5$ balances factual and semantic priorities equally, while lower values prioritize factual accuracy and higher values emphasize semantic coherence.

To illustrate how this metric operates in practice, consider a simple example:

\begin{itemize}
\item Query: ``What is the capital of France?''
\item Response A: ``Paris is the capital of France.''
\item Response B: ``The capital of France is Paris.''
\end{itemize}
Both responses contain the same fact (Paris is France's capital), so $d^w_{\mathrm{set}} = 0$ for both. However, $d_{\mathrm{emb}}$ may differ slightly due to wording differences. The hybrid metric $d_{\mathrm{sem}}$ captures this nuance while still recognizing factual equivalence. 
This example demonstrates how our metric balances factual correctness with semantic nuance, providing a more comprehensive measure of information fidelity than either component alone.

\begin{definition}[Cumulative Distortion]
\label{def:cumulative_distortion}
Let $T$ be the total number of tool calls in an MCP interaction. The cumulative distortion is defined as:
\begin{align*}
  D(T) \;=\; \sum_{t=1}^T \Delta_t,
\end{align*}
where $\Delta_t$ is the step-wise semantic distortion at step $t$ as defined in Equation~\eqref{eq:d_sem}.
\end{definition}

\subsection{Modeling Assumptions}
\label{subsec:assumptions}

Our Information Fidelity Framework requires standard technical assumptions on normalization and continuity (see Appendix~\ref{app:technical_assumptions}), ensuring our distortion metric is well-defined and bounded in [0,1]. We now state the substantive modeling assumptions.

\begin{definition}[Effective Branching Factor]
\label{def:branching_factor}
The effective branching factor $B \ge 1$ denotes the maximum number of distinct future queries that can be directly influenced by a single response within the interaction graph. Sequential chains satisfy $B=1$, while tree-structured interactions have $B$ equal to the maximum number of children per node.
\end{definition}

\begin{assumption}[Bounded Branching]
\label{assump:bounded_branching}
Let $B$ be the maximum effective branching factor. Then $\beta B < 1$.
\end{assumption}

Here, the branching factor $B$ represents the maximum number of future queries potentially influenced by a single response. This condition is \emph{sufficient} to establish our theoretical guarantees, ensuring influence decays across steps. Notably, this is only a \emph{sufficient} but not \emph{necessary} condition for establishing the concentration bound.

\begin{assumption}[Response Stability]
\label{assump:response_stability}
There exists $\alpha \leq 1$ such that small input perturbations induce only continuous changes in the LLM's response distribution, governed by an influence function $\phi(i,j)$.
\end{assumption}

Modern LLMs exhibit a degree of robustness to minor input variations, and this assumption captures that property. Without response stability, a single character change could dramatically alter all subsequent reasoning, making error bounds impossible.

\begin{assumption}[Temporal Decay Structure]
\label{assump:decay_structure}
For any $t<T$ and any two execution paths that coincide through time $t$ but differ at step $t+1$, there exist constants $\alpha \in [0,1]$, $\beta \in [0,1)$, and branching factor $B \geq 1$ such that distortions at future steps can be coupled to satisfy
\[
|\Delta_{t+k} - \Delta'_{t+k}| \leq \alpha(\beta B)^{k-2}, \quad k \geq 2,
\]
almost surely under appropriate coupling. Additionally, $|\Delta_{t+1} - \Delta'_{t+1}| \leq 1$ by Lemma~\ref{lemma:bounded_distortion}. If periodic re-grounding occurs every $m$ steps, then $\Delta_{t+k} = \Delta'_{t+k}$ for all $k > m$.
\end{assumption}

\begin{table*}[t]
\centering
\begin{tabular}{lcc}
\toprule
\textbf{Response Pair} & \textbf{Fact Match} & \textbf{Embedding Distance} \\
\midrule
``Apple's revenue was \$94.8B`` vs. ``Apple reported \$94.8B in revenue`` & 1.0 & 0.17 \\
``The GDP grew by 2.5\% in Q2`` vs. ``Second quarter GDP expansion was 2.5\%`` & 1.0 & 0.21 \\
``The meeting is on Monday at 3PM`` vs. ``Monday at 3PM is when we'll meet`` & 1.0 & 0.35 \\
\bottomrule
\end{tabular}
\caption{Embedding distances for responses with identical facts but different wording}
\label{tab:embedding_distances}
\end{table*}

\begin{remark}[Modeling framework justification]
\label{rem:framework_justification}
This temporal decay model captures three key properties of LLM-tool systems: (i) exponential decay reflecting finite effective memory~\cite{press2022train}, (ii) bounded propagation preventing error amplification, and (iii) reset through periodic re-grounding—standard practice in production deployments. The framework connects to established coupling techniques from Markov chain mixing~\cite{levin2017markov} and adaptive data analysis~\cite{dwork2015preserving}, providing sufficient conditions for concentration bounds. Section~\ref{sec:experiments} validates the decay rate $\beta$ through autocorrelation diagnostics.
\end{remark}

\begin{figure}[t]
  \centering
  
  \begin{tikzpicture}[node distance=0.95cm,>=stealth]
    \tikzstyle{query}=[circle, draw=black, fill=blue!20, minimum size=0.65cm]
    \tikzstyle{response}=[circle, draw=black, fill=green!20, minimum size=0.65cm]
    \tikzstyle{dep}=[->, thick]
    \node (q1) [query]                {\(Q_1\)};
    \node (r1) [response, right of=q1] {\(R_1\)};
    \node (q2) [query, right of=r1]  {\(Q_2\)};
    \node (r2) [response, right of=q2] {\(R_2\)};
    \node (q3) [query, right of=r2]  {\(Q_3\)};
    \node (r3) [response, right of=q3] {\(R_3\)};
    
    \node (dots) [right of=r3, node distance=1cm] {\(\ldots\)};
    
    \node (qT) [query, right of=dots, node distance=1cm] {\(Q_T\)};
    \node (rT) [response, right of=qT] {\(R_T\)};
   
    \draw[dep] (r1) -- node[below] {\(\phi(1,2)\)} (q2);
    \draw[dep] (r2) -- node[below] {\(\phi(2,3)\)} (q3);
    \draw[dep] (r3) -- (dots);
    \draw[dep] (dots) -- (qT);
    \draw[dep, dashed] (r1) to[bend left=45] node[above,font=\small]{\(\beta^2\)} (q3);
    \draw[dep, dashed] (r2) to[bend left=45] node[above,font=\small]{\(\phi(2,T)\)} (qT);

  \end{tikzpicture}
  \caption{Dependency graph for MCP interactions. Solid arrows indicate direct influence \(\phi(i,i+1)=\beta\), and the dashed arrow shows long‐range decay \(\phi(i,j)=\beta^{j-i}\). This exponential decay structure enables us to prove that distortion deviations remain sublinear, even with adaptive queries.}

  \label{fig:dependency_structure}
\end{figure}
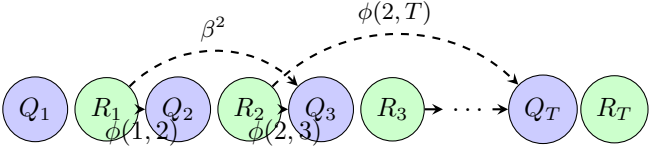

\subsection{Modeling as an Adaptive Stochastic Process}
\label{subsec:stochastic_process}

The inherently sequential nature of tool-augmented reasoning requires careful modeling of dependencies between successive interactions. We formalize this as a sequence of query-response pairs \(\{(Q_t,R_t)\}_{t=1}^T\), where each new query can depend on all previous interactions:
\[
  Q_t \;\in\;\sigma\bigl(Q_1,R_1,\dots,Q_{t-1},R_{t-1}\bigr).
\]
This adaptive formulation is essential for capturing realistic LLM behavior, where each tool call builds on previous responses.

Figure~\ref{fig:dependency_structure} illustrates this dependency structure. A key insight is that influence follows an exponential decay pattern described by Definition \ref{def:influence}. This pattern aligns with how information naturally attenuates in sequential reasoning: recent statements have stronger impact than distant ones, and verification steps effectively reset error chains by grounding the model in authoritative data.

\subsection{Properties of Semantic Distortion}
\label{subsec:distortion_properties}
Our semantic distortion metric exhibits crucial properties that enable concentration analysis. First, it distinguishes lexical variation from semantic equivalence:

\begin{claim}[Semantic Sensitivity]
\label{lemma:semantic_sensitivity}
If two responses contain identical facts but differ lexically, then
\[
  d^w_{\mathrm{set}}(R_t,\mathcal{I}_t)
  = d^w_{\mathrm{set}}(R'_t,\mathcal{I}_t),
  \quad
  d_{\mathrm{emb}}(R_t,\mathcal{I}_t)
  \neq d_{\mathrm{emb}}(R'_t,\mathcal{I}_t),
\]
ensuring \(d_{\mathrm{sem}}\) captures semantic nuance.
\end{claim}

\begin{proof}
If $R_t$ and $R'_t$ have identical fact coverage, then 
\[
   \text{extract}(R_t) = \text{extract}(R'_t) = \hat{\mathcal{I}}_t.
\]
Therefore, $d_{set}^w(R_t, \mathcal{I}_t) = d_{set}^w(R'_t, \mathcal{I}_t)$.

However, different wording produces different embeddings: $\text{embed}(R_t) \neq \text{embed}(R'_t)$, causing $d_{emb}(R_t, \mathcal{I}_t) \neq d_{emb}(R'_t, \mathcal{I}_t)$. For $\lambda > 0$, this ensures $d_{sem}(R_t, \mathcal{I}_t) \neq d_{sem}(R'_t, \mathcal{I}_t)$. Table~\ref{tab:embedding_distances} shows empirical evidence: embedding distances vary 0.15--0.35 in cosine space even when factual content is identical.
\end{proof}

Second, the metric varies continuously with response perturbations, ensuring small edits induce proportionally small distortion changes (formalized in Lemma~\ref{lemma:continuity}, Appendix~\ref{app:proof_continuity}).




\subsection{Martingale Construction for Concentration}
\label{subsec:martingale}

To bound the cumulative distortion \(D(T)=\sum_{t=1}^T \Delta_t\), we construct a martingale sequence whose increments we can control with classical concentration inequalities.

\begin{definition}[Distortion Martingale]\label{def:distortion_martingale}
Define
\[
  Z_t \;=\;\mathbb{E}\bigl[D(T)\,\bigm|\,\mathcal{F}_t\bigr],
  \quad t=0,1,\dots,T,
\]
where \(Z_0=\mathbb{E}[D(T)]\) and \(Z_T=D(T)\).
\end{definition}

This forms a martingale because each $Z_t$ uses all available information up to time $t$ to predict the final distortion $D(T)$. The martingale property ensures that future distortions are conditionally unbiased relative to our current estimate---mathematically, $\mathbb{E}[Z_{t+1}|\mathcal{F}_t] = Z_t$. This property is crucial as it allows us to apply concentration inequalities to bound how far the actual distortion can deviate from its expectation.

Intuitively, $Z_t$ represents our best prediction of the final cumulative distortion given all information available after $t$ interactions; hence observing the next response cannot systematically bias that prediction. As $t$ increases, $Z_t$ incorporates more information and converges to the actual distortion $D(T)$ when all interactions are observed.

In the next section, we show that the increments \(Z_{t+1}-Z_t\) satisfy a bounded-difference property, which we combine with Azuma's inequality to derive high-probability concentration bounds on how much the actual distortion $D(T)$ can
deviate from its linearly growing expected value, ensuring predictable system behavior despite adaptive querying.

\section{Concentration for Adaptive Querying}
\label{sec:concentration}

Having established our information fidelity framework, we now address the central theoretical question: how does distortion accumulate across sequential tool calls? While Definition~\ref{def:cumulative_distortion} specifies that $D(T) = \sum_{t=1}^T \Delta_t$, each $\Delta_t$ can depend on all previous interactions through the adaptive process formalized in Section~\ref{subsec:stochastic_process}. This adaptive querying could, in principle, allow errors to compound arbitrarily. Our main result shows otherwise: under the temporal decay structure of Assumption~\ref{assump:decay_structure}, cumulative distortion remains tightly concentrated around its expectation (which grows at most linearly, with empirically constant per-step rates), with deviations bounded by $O(\sqrt{T})$ despite full adaptivity.

Our main insight is that despite adaptive querying, where each new prompt can depend arbitrarily on all previous interactions, the cumulative distortion $D(T)=\sum_{t=1}^{T}\Delta_t$ remains tightly concentrated around its expectation. This concentration emerges from the exponential decay of influence across steps, creating a form of ``effective independence`` that enables powerful probabilistic guarantees.

\subsection{The Bounded-Difference Property}

The following lemma establishes how much a single new response can affect our expectation of the total distortion under the temporal decay structure. The constant $C^* = \frac{\alpha}{1-\beta B}$ quantifies the maximum possible cumulative influence that a single response can have on all future queries, with $\alpha$ capturing response stability (Assumption~\ref{assump:response_stability}) and $B$ accounting for the branching factor (Assumption~\ref{assump:bounded_branching}). This uniform bound allows us to derive clean concentration results regardless of a query's position in the sequence.

\begin{lemma}[Bounded Doob increments]
\label{lemma:bounded_differences}
Let $Z_t = \mathbb{E}[D(T) \mid \mathcal{F}_t]$ be the distortion Doob martingale (Definition~\ref{def:distortion_martingale}). Under Assumptions~\ref{assump:bounded_branching} and~\ref{assump:decay_structure}, we have, almost surely,
\[
|Z_{t+1} - Z_t| \leq c_\star \quad \text{for all } t < T,
\]
with
\[
c_\star = 1 + C^* = 1 + \frac{\alpha}{1-\beta B}
\quad \text{when } \beta B < 1.
\]
If periodic re-grounding enforces a finite horizon $m$, then 
\[
c_\star = 1 + \alpha \cdot \frac{1 - (\beta B)^m}{1 - \beta B}
\]
\end{lemma}

\begin{proof}[Proof Sketch]
The martingale increment decomposes as $Z_{t+1} - Z_t = (\Delta_{t+1} - \mathbb{E}[\Delta_{t+1}|\mathcal{F}_t]) + \sum_{j=t+2}^T(\mathbb{E}[\Delta_j|\mathcal{F}_{t+1}] - \mathbb{E}[\Delta_j|\mathcal{F}_t])$. The first term is bounded by 1 via Lemma~\ref{lemma:bounded_distortion}. For the tail sum, the pathwise coupling from Assumption~\ref{assump:decay_structure} ensures $|\Delta_j - \Delta'_j| \leq \alpha(\beta B)^{j-t-2}$ almost surely for coupled executions, implying each conditional expectation difference is bounded by the same factor. Summing the geometric series $\sum_{j=t+2}^T \alpha(\beta B)^{j-t-2} \leq \alpha/(1-\beta B) = C^*$ completes the bound. Full details in Appendix~\ref{app:proof_bounded_differences}.
\end{proof}

This bounded-difference property is the linchpin enabling concentration analysis: it ensures that even though each new response can influence all future queries through the adaptive chain, that influence diminishes rapidly enough to maintain exponential concentration around the expected cumulative distortion. 

\subsection{High-Probability Concentration Results}

\begin{theorem}[High-probability distortion bound]
\label{thm:mcp_concentration}
Let $C^* = \frac{\alpha}{1-\beta B}$ and $\gamma^* = 2C^* + (C^*)^2$. For any $\eta\in(0,1)$,
\[
  \Pr\!\left[D(T)-\mathbb{E}D(T)\;\ge\;
  \sqrt{\,2\,T(1+\gamma^*)\,\ln\tfrac1\eta\,}\right]
  \;\le\;\eta.
\]
\end{theorem}

\begin{proof}[Proof Sketch]
By Lemma~\ref{lemma:bounded_differences}, the distortion martingale has \\ bounded increments $|Z_{t+1} - Z_t| \leq c_\star = 1 + C^*$. Summing $c_\star^2 = (1 + C^*)^2 = 1 + 2C^* + (C^*)^2$ over $T$ steps yields $T(1 + \gamma^*)$. Azuma's inequality then gives the stated high-probability bound. The complete proof appears in Appendix~\ref{app:proof_mcp_concentration}.
\end{proof}

The correction factor $\gamma^*$ quantifies how dependencies in our adaptive process inflate the variance term compared to the independent case. With $C^* = \frac{\alpha}{1-\beta B}$ representing the maximum cumulative influence of any response on all future queries, $\gamma^* = 2C^* + (C^*)^2$ captures both the linear interaction between bounded differences and the quadratic effect of compounded influences. Together, they show that even with adaptive queries, cumulative distortion grows at most linearly in expectation (bounded by T since $\Delta_t \in [0,1]$), with empirically observed constant per-step rates and sublinearly bounded deviations.
This concentration property rules out both 
exponential error blowup and chaotic variance growth.
The practical implications of this result are captured in the following corollaries:

\begin{corollary}[Sub-linear deviation]
\label{cor:sublinear}
With probability at least $1-\eta$,
\(
  D(T)=\mathbb{E}D(T)+O\!\bigl(\sqrt{T\ln(1/\eta)}\bigr).
\)
\end{corollary}

This result directly addresses fears about chaotic error compounding in LLM reasoning chains. Rather than exhibiting exponential divergence, cumulative distortion remains tightly concentrated around a predictable linear trajectory. The $O(\sqrt{T})$ deviation bound guarantees that uncertainty grows sublinearly: deviations from expected distortion scale as $\sqrt{T}$, not $T$. For $T=10$ steps, the high-probability deviation is bounded by $\sqrt{10} \approx 3.2$ times the single-step deviation constant, ensuring predictable system behavior.

\begin{corollary}[Effective information horizon]
\label{cor:info_horizon}
The query-level influence function $\phi(i,j) = \beta^{j-i}$ (Definition~\ref{def:influence}) 
decays exponentially, with influence dropping below threshold $\varepsilon$ after 
$\mathcal{H}_\varepsilon = \lceil\log(\varepsilon)/\log(\beta)\rceil$ steps. 
For $\beta$ close to $1$, this simplifies to approximately $\lceil -\ln(\varepsilon)/(1-\beta)\rceil$. 
The \emph{cumulative} propagation of a single error across all future queries is bounded by $C^* = \alpha/(1-\beta B)$, which quantifies the maximum total influence (amplitude). Together, $\mathcal{H}$ (range) and $C^*$ (magnitude) determine how far and how strongly errors can propagate.
\end{corollary}

This corollary introduces the concept of an ``effective information horizon``: the maximum distance over which errors meaningfully propagate through the query chain. The horizon $\mathcal{H}$ is set by the decay rate $\beta$ of $\phi(i,j)$ (range), while the amplitude of propagated errors is scaled by $C^*$, which incorporates both response stability ($\alpha$) and branching ($B$) (magnitude). For typical values (e.g., $\beta = 0.7$), this horizon is approximately 9 steps, supporting periodic re-grounding in long reasoning chains.
Complete proofs of above results, along with refinements using tighter concentration inequalities, appear in Appendix~\ref{app:proof_mcp_concentration}, \ref{app:proof_sublinear} and~\ref{app:proof_info_horizon}, respectively.


\section{Experiments}\label{sec:experiments}

Having established theoretical guarantees for information fidelity, we validate two central predictions: (1) Does cumulative distortion exhibit linear expected growth with sublinear concentration as claimed by Corollary~\ref{cor:sublinear}? 
(2) How do framework parameters ($\beta$, $\lambda$) affect empirical distortion? 
We address these through systematic experiments across Qwen2-7B-Instruct (Qwen) ~\cite{qwen2} and Llama-3-8B-Instruct (Llama)~\cite{llama3modelcard}, with Mistral-7B-Instruct-v0.3 (Mistral)~\cite{jiang2023mistral} in Appendix~\ref{app:exp_results_mistral}.


\paragraph{MCP Implementation.}
All experiments use identical MCP tools: a knowledge retrieval tool operating over a deterministic cached corpus covering eight domains (history, science, technology, arts, sports, geography, literature, mathematics), and a financial data tool providing schema-validated access to market snapshots. Both enforce strict schema validation and maintain context continuity, faithfully implementing MCP specifications. To isolate error accumulation dynamics from tool variability, we employ deterministic, cached responses. This controlled setting establishes fundamental bounds that serve as a baseline for understanding systems with stochastic tools.

\paragraph{Query Generation and Distortion Measurement.}
Each chain begins with a seed question from domain-specific templates. Subsequent queries are generated adaptively with probability $\beta$, implementing the influence function $\phi(i,j) = \beta^{j-i}$ from Definition~\ref{def:influence}. Our experimental chains maintain linear structure with branching factor $B=1$, satisfying Assumption~\ref{assump:bounded_branching}'s condition $\beta B < 1$. For each response $R_t$, we extract ideal fact set $\mathcal{I}_t$ from ground truth and compute hybrid semantic distortion via Eq.~\eqref{eq:d_sem}. Cumulative distortion $D(T) = \sum_{t=1}^T \Delta_t$ is tracked across each chain, with statistics aggregated over multiple independent trials.

\paragraph{Theoretical Bounds and Model Validation.}
We construct calibrated envelopes by: (1) estimating expected per-step distortion $\hat{r}$ from mean first-step distortion, (2) extrapolating linearly to $\hat{\mathbb{E}}[D(T)] = T \cdot \hat{r}$, and (3) adding theoretical deviation $\sqrt{2(1+\hat{\gamma})T\ln(1/\delta)}$ where $\hat{\gamma}$ is computed from empirical dependency strength (Appendix~\ref{app:calibrated_envelopes}). These bounds instantiate Theorem~\ref{thm:mcp_concentration} with 95\% confidence ($\delta=0.05$). To validate Assumption~\ref{assump:decay_structure}, we compute empirical autocorrelations $\hat{\rho}(k) = \mathrm{Corr}(\Delta_t, \Delta_{t+k})$ and fit decay rate $\hat{\beta} = \arg\min_\beta \sum_k (\hat{\rho}(k) - \beta^k)^2$, yielding $\hat{\beta} \in [0.68, 0.71]$ across architectures. Full details appear in Appendix~\ref{app:full_exp_details_new}.

\paragraph{Baseline Validation.}
Under standard conditions ($\beta=0.7$, $\lambda=0.5$, $T=10$, 50 chains per model), Figure~\ref{fig:ten_step} reveals striking consistency between theory and practice. Qwen2-7B achieves $D(10) = 5.26 \pm 0.34$ versus calibrated envelope 9.15 (safety margin $1.74\times$); Llama-3-8B records $D(10) = 4.92 \pm 0.46$ with envelope 8.90 (margin $1.81\times$). The consistent per-step rate of $\approx 0.5$ across all chains confirms linear expected growth, while safety margins $>1.7\times$ validate our $O(\sqrt{T})$ concentration bounds from Corollary~\ref{cor:sublinear}.

\begin{figure}[t]
  \centering
  \includegraphics[width=0.9\linewidth]{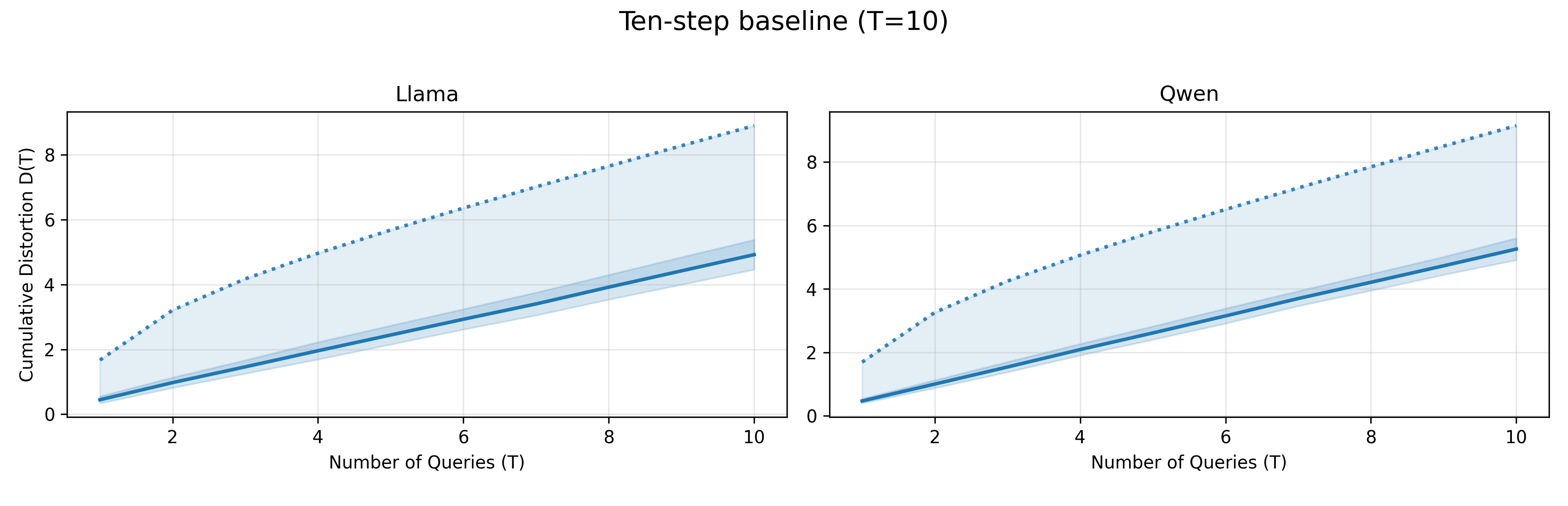}
  \caption{\textbf{Baseline distortion accumulation over 10 tool calls.} Cumulative distortion with $\beta = 0.7$, $\lambda = 0.5$ (50
chains/model). Solid: empirical mean $\pm 1\sigma$; dotted: high-probability envelopes (Theorem~\ref{thm:mcp_concentration}, 95\% confidence). Distortion grows linearly at $\approx 0.5$ per step with deviations tightly concentrated around the linear trend, validating $O(\sqrt{T})$ concentration bounds.}

  \label{fig:ten_step}
\end{figure}

\begin{figure}[t]
  \centering
  \includegraphics[width=0.9\linewidth]{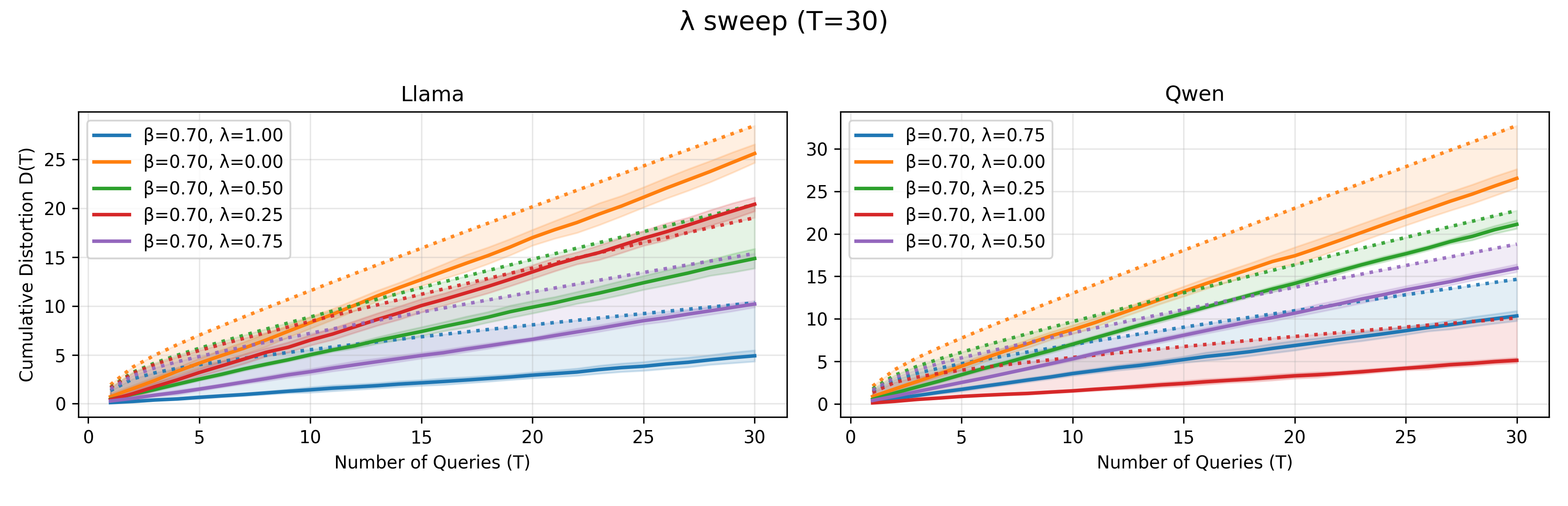}
  \caption{\textbf{Lambda sweep: semantic weighting effect.} Distortion for $\lambda \in \{0, 0.25, 0.5, 0.75, 1.0\}$ at $T=30$, $\beta=0.7$ (8 chains/config). Solid: empirical mean $\pm 1\sigma$; dotted: calibrated bounds.}
  \label{fig:lambda_sweep}
\end{figure}

\paragraph{Extended Validation.}
Experiments across $T=60$ steps and extreme dependencies ($\beta \in \{0.5, 0.7, 0.9, 0.95, 0.98\}$) demonstrate {predictable linear accumulation} with $O(\sqrt{T})$ concentration bounds maintaining consistent safety margins ($1.10\times$--$1.55\times$). Even at $\beta=0.98$ approaching the critical point, the system avoids exponential failure modes, showing that high dependencies inflate the variance bounds rather than the mean distortion rate (Figures~\ref{fig:long_chain},~\ref{fig:high_beta}). 

\begin{figure}[t]
  \centering
  \includegraphics[width=0.9\linewidth]{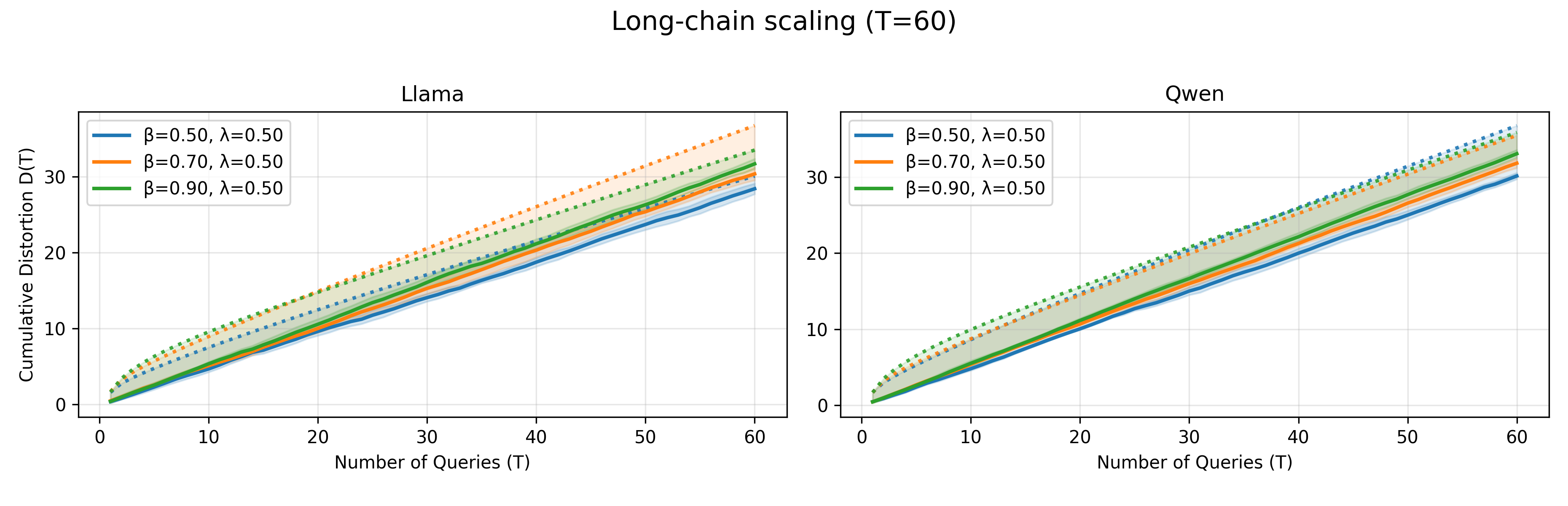}
  \caption{\textbf{Extended chain scaling.} Distortion for $\beta \in \{0.5, 0.7, 0.9\}$ at $T=60$, $\lambda=0.5$ (6 chains/config). Solid: mean $\pm 1\sigma$; dotted: calibrated bounds.}
  \label{fig:long_chain}
\end{figure}

\begin{figure}[t]
  \centering
  \includegraphics[width=0.9\linewidth]{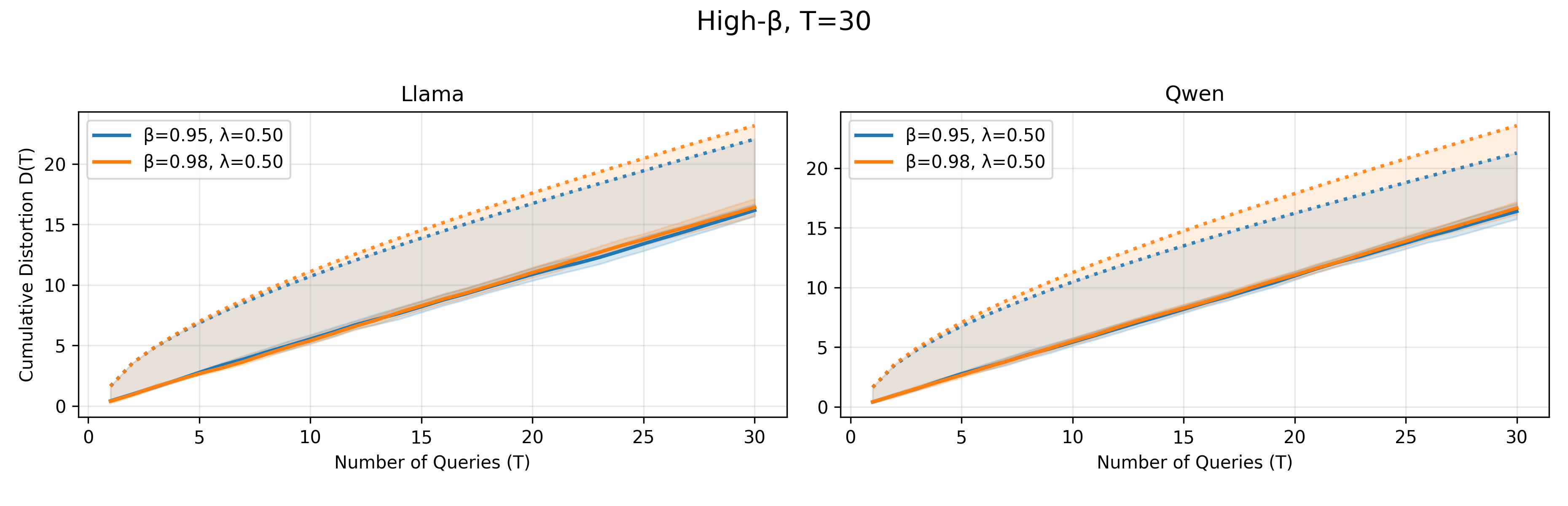}

\caption{\textbf{High-dependency stress test.} Distortion for $\beta \in \{0.95, 0.98\}$ at $T=30$, $\lambda=0.5$ (6 chains/config). Solid: mean $\pm 1\sigma$; dotted: calibrated bounds. At $\beta=0.98$ approaching the critical point $\beta B=1$, empirical curves stay well below theoretical envelopes.}

  \label{fig:high_beta}
\end{figure}


\paragraph{Summary.}
Our experimental evaluation demonstrates three key results: 
(1) \emph{Linear growth with sublinear concentration holds robustly}—cumulative distortion grows linearly at constant per-step rates ($0.50$--$0.55$) across chain lengths $T \in [10, 60]$ and dependency strengths $\beta \in [0.5, 0.98]$, with empirical trajectories consistently within $O(\sqrt{T})$ theoretical envelopes. 
(2) \emph{Framework parameters offer actionable design levers}—increasing semantic weight $\lambda$ from 0 to 1 reduces distortion by 80\%, while stronger dependencies ($\beta$) increase per-step rates modestly ($\sim$4\% from 0.7 to 0.9). 
(3) \emph{Architectural robustness validates theoretical generality}—three LLM architectures (Qwen2-7B, Llama-3-8B, Mistral-7B) exhibit nearly identical distortion patterns, confirming that concentration bounds depend on dependency structure ($\beta$, $B$) and metric properties ($\lambda$), not internal model mechanisms.

\section{Discussion and Implications}\label{sec:discussion}
Our Information Fidelity Framework shows that tool-using LLM agents can achieve provable reliability—cumulative distortion grows at most linearly (bounded by T, with empirically constant per-step rates) with sublinear concentration bounds on deviations, ruling out exponential error blowup—through deliberate architectural choices:

\paragraph{Architecting for Reliability.}
To realize the bounds in practice, agents should be engineered so that older context matters less than recent context (e.g., via sliding or recency-biased windows), tool outputs are predictable (schema validation, caching, and bounded stochasticity), and dependencies are reset by re-grounding periodically on authoritative sources. Schedule re-grounding every $\mathcal{H}_\varepsilon=\lceil\log(\varepsilon)/\log(\beta)\rceil$ steps for chosen threshold $\varepsilon$. Limit fan-out via gating or serialization to keep branching factor $B$ small. These controls ensure $\beta B<1$, enabling our concentration bounds.

\paragraph{Tuning and Monitoring in Production.}
Tune $\lambda$ balancing factual precision ($\lambda\in[0.3,0.5]$ for safety-critical) versus semantic coherence ($\lambda\in[0.6,0.8]$ for conversational). Distortion decreases sharply with $\lambda$ up to 0.75, with diminishing returns beyond (Fig.~\ref{fig:lambda_sweep}). For typical values ($\beta=0.7$, $\varepsilon=0.05$), re-ground every $\sim$9 steps. 
Fit $\hat{\beta}$ from autocorrelations to adapt verification cadence dynamically.

\paragraph{Limitations and Extensions.}
Our theoretical framework analyzes single-agent, sequential tool interactions with deterministic tools that enable rigorous concentration bounds. Natural extensions include analyzing multi-agent error propagation, relaxing assumptions, incorporating stochastic tool responses, and adapting to extreme dependency regimes. The metric emphasizes fact recall over precision; alternative weighting schemes could prioritize different error modes. Adaptive policies and online distortion estimation represent promising directions for future work.

\section{Conclusion}
\label{sec:conclusion}

As LLM agents mediate high-stakes decisions, formal reliability guarantees transition from curiosities to necessities. Our Information Fidelity Framework establishes that catastrophic error accumulation is avoidable: bounded contexts, stable tools, and periodic verification provably limit distortion to linear growth $O(T)$ with deviations 
concentrated at $O(\sqrt{T})$, preventing exponential error blowup while 
maintaining predictable system behavior. The theory prescribes, not merely describes—providing practitioners with actionable design principles grounded in rigorous concentration bounds.

\bibliography{aaai2026}

\appendix
\onecolumn

\section{Model Context Protocol Primer}\label{app:mcp_primer}
The Model Context Protocol (MCP) is an open standard introduced by Anthropic in late 2024 that standardizes how AI applications, particularly large language models (LLMs), connect with external data sources and tools. MCP addresses the fundamental limitation of conventional LLMs---their isolated nature and inability to access real-time information without custom integrations for each data source. By providing a universal interface for AI-data interactions, MCP transforms what would traditionally be an $M\times N$ integration problem ($M$ different AI applications connecting to $N$ different tools) into a more manageable $M+N$ approach~\cite{anthropic2024mcp}.

\begin{figure}[hb]
  \centering
    \includegraphics[width=0.6\linewidth]{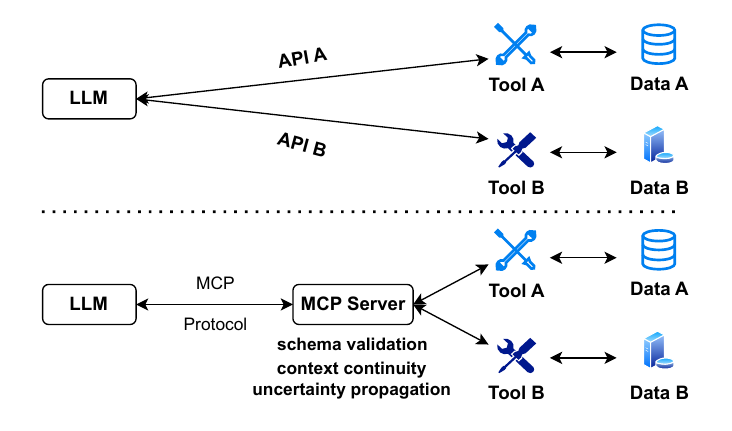}
  \caption{Comparison of ad-hoc versus MCP-based tool integration. \textbf{(Top)} An LLM makes separate, custom API calls to Tool A and Tool B, each requiring bespoke handling and context management. \textbf{(Bottom)} The LLM communicates via a standard JSON-RPC interface to an MCP server, which enforces schema validation, maintains context continuity, propagates uncertainty estimates, and routes calls to multiple tools, returning structured responses back to the model.}
  \label{fig:app_mcp_workflow}
\end{figure}

MCP extends the JSON-RPC 2.0 protocol, a stateless, lightweight remote procedure call framework that uses JSON for data serialization. While JSON-RPC defines general request-response structures, MCP introduces LLM-specific features: schema validation to ensure data consistency, context continuity to maintain state across interactions, and uncertainty handling to quantify response reliability. These enhancements make MCP suitable for dynamic AI workflows.

In an MCP interaction, the LLM sends a JSON-RPC request to an MCP server, which validates the request, routes it to the appropriate tool (e.g., a database or API), and returns a structured response. The request includes standard JSON-RPC fields: \texttt{jsonrpc} (version), \texttt{method} (tool function), \texttt{params} (parameters), and \texttt{id} (response matching). MCP adds optional fields like \texttt{context} for session tracking and \texttt{uncertainty} in responses to indicate confidence.

For example, an MCP request to retrieve a stock price might be:
\begin{verbatim}
{
  ``jsonrpc``: ``2.0``,
  ``method``: ``get_stock_price``,
  ``params``: {``symbol``: ``AAPL``},
  ``id``: 1,
  ``context``: {``session_id``: ``abc123``}
}
\end{verbatim}
The response could be:
\begin{verbatim}
{
  ``jsonrpc``: ``2.0``,
  ``result``: {``price``: 150.25, ``timestamp``: ``2025-05-15T02:02:00+08:00``},
  ``uncertainty``: 0.01,
  ``id``: 1
}
\end{verbatim}
Here, \texttt{uncertainty} reflects the confidence in the price data.

MCP's standardized interface reduces integration complexity, ensures consistent context management, and supports reliable tool-augmented LLM systems, making it a cornerstone for modern AI applications.

\section{Full Proofs of Theoretical Results}
\label{app:full_proofs}

This section provides complete proofs for the theoretical results presented in the main paper.
\subsection{Technical Assumptions and Auxiliary Lemmas}
\label{app:technical_assumptions}
We begin by stating technical assumptions and supporting lemmas omitted from the main text due to space constraints. These conditions are \emph{sufficient} for our proofs to flow through and represent standard properties satisfied by modern embedding models and fact-extraction systems. While these are sufficient conditions that ensure mathematical rigor, they are not claimed to be necessary—our empirical validation (Section~\ref{sec:experiments}) demonstrates that the predicted behavior holds under practical implementations that approximately satisfy these requirements.

\begin{assumption}[Weight Normalization]
\label{assump:weight_norm}
For each ideal fact set $\mathcal{I}_t$, $\sum_{f\in\mathcal{I}_t}w(f)=1$.
\end{assumption}

This standard normalization ensures our distortion metric remains bounded between 0 and 1, making quantitative comparisons meaningful. It can always be satisfied by appropriately scaling the weight function without affecting the analysis~\cite{gao2021simcse,opitz2022sbert}.

\begin{assumption}[Embedding Normalization]
\label{assump:emb_norm}
All embeddings satisfy $\|\mathrm{embed}(x)\|=1$.
\end{assumption}

Most modern embedding models, including SimCSE~\cite{gao2021simcse} and SBERT~\cite{opitz2022sbert}, normalize output vectors by default, making this a natural assumption that simplifies semantic similarity calculations.

\begin{assumption}[Embedding Regularity]
\label{assump:emb_reg}
For any unit vector $a$ and any texts $x,y$, the embedding function satisfies
\[
\bigl|\cos\bigl(\mathrm{embed}(x),a\bigr) - \cos\bigl(\mathrm{embed}(y),a\bigr)\bigr|
\;\leq\; \bigl\|\mathrm{embed}(x)-\mathrm{embed}(y)\bigr\|.
\]
\end{assumption}

This property follows directly from the Cauchy--Schwarz inequality for unit-normalized embeddings and requires no empirical verification. It ensures that semantic similarity varies continuously in embedding space, which is essential for establishing continuity of our distortion metric (Lemma~\ref{lemma:continuity}).

\begin{assumption}[Extraction Stability]
\label{assump:extraction_stability}
The fact-extraction process is robust to edits below a threshold $\tau>0$.
\end{assumption}

This assumption reflects the discreteness of fact extraction: small edits shouldn't change which facts are recognized.

\begin{lemma}[Continuity of Distortion]
\label{lemma:continuity}
Let $\tau > 0$ be a threshold such that fact extraction is unchanged for edits below $\tau$ (Assumption~\ref{assump:extraction_stability}). Then for any two responses $R, R'$,
\begin{align*}
\bigl|d_{\mathrm{sem}}(R,\mathcal{I}_t) - d_{\mathrm{sem}}(R',\mathcal{I}_t)\bigr| \leq 
(1-\lambda) \cdot \mathbf{1}\{\text{extracted facts differ}\}
\\ + \frac{\lambda}{2}\,\bigl\|\mathrm{embed}(R)-\mathrm{embed}(R')\bigr\|.
\end{align*}
In particular, if the extracted facts are identical (e.g., when the edit distance is $< \tau$), the change is bounded by $\frac{\lambda}{2}\|\mathrm{embed}(R) - \mathrm{embed}(R')\|$.
\end{lemma}

\begin{remark} 
If in a particular deployment one can verify an encoder-specific inequality $\|\mathrm{embed}(x) - \mathrm{embed}(y)\| \leq L_{\mathrm{edit}} \cdot d_{\mathrm{edit}}(x,y)$ for some constant $L_{\mathrm{edit}}$, then Lemma~\ref{lemma:continuity} implies the bound $\frac{\lambda L_{\mathrm{edit}}}{2} \cdot d_{\mathrm{edit}}(R, R')$ for the semantic term. We do not assume this globally, as neural embeddings are not generally Lipschitz continuous with respect to edit distance. 
\end{remark}

\begin{lemma}[Bounded Distortion]
\label{lemma:bounded_distortion}
For all \(R_t\) and \(\mathcal{I}_t\), 
\[
  0 \;\le\; d_{\mathrm{sem}}(R_t,\mathcal{I}_t)\;\le\;1
\]
\end{lemma}

\begin{proof}
For the set-based component, due to the weight normalization condition (Assumption \ref{assump:weight_norm}) and the fact that $\hat{\mathcal{I}}_t \cap \mathcal{I}_t \subseteq \mathcal{I}_t$, we have $0 \leq d_{set}^w(R_t, \mathcal{I}_t) \leq 1$.

For the embedding component, since $-1 \leq \cos(u, v) \leq 1$ for any unit vectors $u, v$ (Assumption \ref{assump:emb_norm}), we have 
\[0 \leq \frac{1 - \cos(\text{embed}(R_t), \text{embed}(\mathcal{I}_t))}{2} \leq 1.
\]

Since $d_{sem}$ is a convex combination of these two bounded components with weight $\lambda \in [0,1]$, we have $0 \leq d_{sem}(R_t, \mathcal{I}_t) \leq 1$.
\end{proof}

\subsection{Proof of Lemma \ref{lemma:continuity} (Continuity of Distortion Metric)}
\label{app:proof_continuity}

\begin{proof}
Let $R$ and $R'$ be two responses. We analyze the bound for each component of the distortion metric separately and then combine them.

\textbf{Set-based component:} The discrete fact extractor has a threshold property from Assumption~\ref{assump:extraction_stability}: extracted fact sets remain identical for small edits but may change beyond edit distance $\tau$. Therefore:
\[
|d^w_{\mathrm{set}}(R, \mathcal{I}_t) - d^w_{\mathrm{set}}(R', \mathcal{I}_t)| = 
\begin{cases}
0 & \text{if extracted facts identical} \\
\leq 1 & \text{if extracted facts differ}
\end{cases}
\]

\textbf{Embedding component:} Both $\mathrm{embed}(R)$ and $\mathrm{embed}(R')$ are unit vectors by Assumption~\ref{assump:emb_norm}. Let $a = \mathrm{embed}(\mathcal{I}_t)$, also a unit vector. The embedding distance is:
\[
d_{\mathrm{emb}}(R, \mathcal{I}_t) = \frac{1 - \cos(\mathrm{embed}(R), a)}{2}.
\]

By Assumption~\ref{assump:emb_reg} (Embedding Regularity):
\begin{align*}
|d_{\mathrm{emb}}(R, \mathcal{I}_t) - d_{\mathrm{emb}}(R', \mathcal{I}_t)| 
&= \frac{1}{2} |\cos(\mathrm{embed}(R), a) - \cos(\mathrm{embed}(R'), a)| \\
&\leq \frac{1}{2} \|\mathrm{embed}(R) - \mathrm{embed}(R')\|.
\end{align*}

\textbf{Combining components:} The hybrid distortion is $d_{\mathrm{sem}} = (1-\lambda) d^w_{\mathrm{set}} + \lambda d_{\mathrm{emb}}$. Therefore:
\begin{align*}
&|d_{\mathrm{sem}}(R, \mathcal{I}_t) - d_{\mathrm{sem}}(R', \mathcal{I}_t)| \\
&\leq (1-\lambda) |d^w_{\mathrm{set}}(R, \mathcal{I}_t) - d^w_{\mathrm{set}}(R', \mathcal{I}_t)| + \lambda |d_{\mathrm{emb}}(R, \mathcal{I}_t) - d_{\mathrm{emb}}(R', \mathcal{I}_t)| \\
&\leq (1-\lambda) \cdot \mathbf{1}\{\text{facts differ}\} + \frac{\lambda}{2} \|\mathrm{embed}(R) - \mathrm{embed}(R')\|.
\end{align*}

When extracted facts are identical (e.g., edit distance $< \tau$), the first term vanishes and only the embedding term remains.
\end{proof}

\subsection{Proof of Lemma \ref{lemma:bounded_differences} (Bounded Doob Increments)}
\label{app:proof_bounded_differences}

\begin{proof}
Fix $t < T$. Consider two coupled execution paths that agree up to time $t$ and may differ at step $t+1$, with future randomness coupled as specified in Assumption~\ref{assump:decay_structure}. Let $\Delta_{t+1}, \Delta_{t+2}, \ldots, \Delta_T$ denote distortions in the first execution and $\Delta'_{t+1}, \Delta'_{t+2}, \ldots, \Delta'_T$ in the second.

The martingale increment decomposes as:
\begin{align*}
Z_{t+1} - Z_t 
&= \mathbb{E}[D(T) \mid \mathcal{F}_{t+1}] - \mathbb{E}[D(T) \mid \mathcal{F}_t] \\
&= \Bigl(\Delta_{t+1} - \mathbb{E}[\Delta_{t+1} \mid \mathcal{F}_t]\Bigr) + \sum_{j=t+2}^T \Bigl(\mathbb{E}[\Delta_j \mid \mathcal{F}_{t+1}] - \mathbb{E}[\Delta_j \mid \mathcal{F}_t]\Bigr).
\end{align*}

\textbf{Bounding the first term:} By Lemma~\ref{lemma:bounded_distortion}, $0 \leq \Delta_{t+1} \leq 1$ almost surely. Therefore:
\[
\bigl|\Delta_{t+1} - \mathbb{E}[\Delta_{t+1} \mid \mathcal{F}_t]\bigr| \leq 1.
\]

\textbf{Bounding the tail contribution:} For $j \geq t+2$, the coupling from Assumption~\ref{assump:decay_structure} ensures that:
\[
|\Delta_j - \Delta'_j| \leq \alpha(\beta B)^{j-(t+1)-1} = \alpha(\beta B)^{j-t-2}
\]
almost surely. This pathwise bound implies:
\begin{align*}
\bigl|\mathbb{E}[\Delta_j \mid \mathcal{F}_{t+1}] - \mathbb{E}[\Delta_j \mid \mathcal{F}_t]\bigr| 
&\leq \mathbb{E}[|\Delta_j - \Delta'_j| \mid \mathcal{F}_t] \\
&\leq \alpha(\beta B)^{j-t-2}.
\end{align*}

Summing over $j = t+2, \ldots, T$:
\begin{align*}
\sum_{j=t+2}^T \bigl|\mathbb{E}[\Delta_j \mid \mathcal{F}_{t+1}] - \mathbb{E}[\Delta_j \mid \mathcal{F}_t]\bigr|
&\leq \sum_{j=t+2}^T \alpha(\beta B)^{j-t-2} \\
&= \alpha \sum_{k=0}^{T-t-2} (\beta B)^k \\
&= \alpha \cdot \frac{1 - (\beta B)^{T-t-1}}{1 - \beta B}.
\end{align*}

Since $\beta B < 1$ by Assumption~\ref{assump:bounded_branching}, we have $1 - (\beta B)^{T-t-1} \leq 1$, yielding:
\[
\sum_{j=t+2}^T \bigl|\mathbb{E}[\Delta_j \mid \mathcal{F}_{t+1}] - \mathbb{E}[\Delta_j \mid \mathcal{F}_t]\bigr|
\leq \frac{\alpha}{1 - \beta B} = C^*.
\]

\textbf{Combining both terms:} Therefore, almost surely:
\[
|Z_{t+1} - Z_t| \leq 1 + C^* = c_\star.
\]

\textbf{Finite-horizon case:} If periodic re-grounding occurs every $m$ steps, then $\Delta_{t+k} = \Delta'_{t+k}$ for all $k > m$ by Assumption~\ref{assump:decay_structure}. The geometric series truncates at $k = m$, giving:
\[
c_\star = 1 + \alpha \sum_{k=0}^{m-1} (\beta B)^k = 1 + \alpha \cdot \frac{1 - (\beta B)^m}{1 - \beta B}.
\]

This completes the proof.
\end{proof}

\subsection{Proof of Theorem \ref{thm:mcp_concentration} (MCP Concentration Bound)}
\label{app:proof_mcp_concentration}

\begin{proof}
By Definition~\ref{def:distortion_martingale}, $\{Z_t\}_{t=0}^T$ forms a martingale with respect to filtration $\mathcal{F}_t$. By Lemma~\ref{lemma:bounded_differences}, the martingale differences are bounded: $|Z_{t+1} - Z_t| \leq c_\star$ almost surely, where $c_\star = 1 + C^*$ with $C^* = \frac{\alpha}{1-\beta B}$.

Applying Azuma's inequality:
\begin{align*}
\Pr(Z_T - Z_0 \geq \delta) \leq \exp\left(-\frac{\delta^2}{2\sum_{t=0}^{T-1} c_\star^2}\right).
\end{align*}

Since $Z_T = D(T)$ and $Z_0 = \mathbb{E}[D(T)]$, we have:
\begin{align*}
\Pr(D(T) - \mathbb{E}[D(T)] \geq \delta) \leq \exp\left(-\frac{\delta^2}{2\sum_{t=0}^{T-1} c_\star^2}\right).
\end{align*}

For our uniform bounded difference constant:
\begin{align*}
\sum_{t=0}^{T-1} c_\star^2 
&= T \cdot (1 + C^*)^2 \\
&= T \cdot (1 + 2C^* + (C^*)^2) \\
&= T(1 + \gamma^*),
\end{align*}
where $\gamma^* = 2C^* + (C^*)^2$.

Therefore:
\begin{align*}
\Pr\left(D(T) - \mathbb{E}[D(T)] \geq \delta\right) \leq \exp\left(-\frac{\delta^2}{2T(1 + \gamma^*)}\right).
\end{align*}

Setting $\delta = \sqrt{2T(1+\gamma^*)\ln\frac{1}{\eta}}$, we obtain:
\begin{align*}
\Pr\left(D(T) - \mathbb{E}[D(T)] \geq \sqrt{2T(1+\gamma^*)\ln\frac{1}{\eta}}\right) \leq \eta.
\end{align*}

This completes the proof of Theorem~\ref{thm:mcp_concentration}.
\end{proof}

\subsection{Proof of Corollary \ref{cor:sublinear} (Sub‐linear deviation)}
\label{app:proof_sublinear}

\begin{proof}
Set
\(
\delta
=
\sqrt{2\,T(1+\gamma^*)\,\ln\!\tfrac1\eta}
\)
in Theorem~\ref{thm:mcp_concentration}.
With probability at least \(1-\eta\),
\[
D(T)
\;=\;
\mathbb{E}D(T)
+
\sqrt{2\,T(1+\gamma^*)\,\ln\tfrac1\eta}
\;=\;
O\!\bigl(\sqrt{T\ln\tfrac1\eta}\bigr).
\]
In practice, exploiting geometric decay structure (Appendix~\ref{app:calibrated_envelopes}) with $B=1$ yields $\hat{\gamma} \ll \gamma^*$, making the leading constant tractable for typical chain lengths.
\end{proof}

\subsection{Proof of Corollary \ref{cor:info_horizon} (Effective Information Horizon)}
\label{app:proof_info_horizon}

\begin{proof}
From our influence function $\phi(i,j) = \beta^{j-i}$, we want to find the number of steps $h$ after which the influence drops below a threshold $\epsilon$:

$\beta^h < \epsilon$

Taking logarithms:
$h \log \beta < \log \epsilon$

Since $\beta < 1$, $\log \beta < 0$, so:
$h > \frac{\log \epsilon}{\log \beta}$

For a small threshold $\epsilon = 0.05$ (95\% decay) and $\beta = 0.7$:
$h > \frac{\log 0.05}{\log 0.7} \approx \frac{-3.00}{-0.36} \approx 8.33$

Therefore, errors propagate effectively over approximately $\mathcal{H}_\varepsilon = \lceil\frac{\log \varepsilon}{\log \beta}\rceil$ steps. For the parameters used in our experiments ($\beta=0.7$), this implies an effective horizon of $\lceil 8.33 \rceil = 9$ steps for 5\% influence, rather than the stricter e-folding time.

For practical implementation, this means that after every $\mathcal{H}_\varepsilon$ steps, the system can be effectively ``reset`` by re-grounding with authoritative information.
\end{proof}

\section{Full Experimental Details}\label{app:full_exp_details_new}

\subsection{MCP Tool Implementation and Deterministic Data Sources}

\paragraph{Design philosophy.}
To isolate distortion dynamics from tool variability, we implement \emph{deterministic, cached tools} that return identical outputs for identical queries across all models and runs. This ensures observed differences stem from model behavior (reasoning, context management, paraphrasing) rather than stochastic tool responses or network latency.

\paragraph{Knowledge Retrieval Tool.}
Implements information lookup over a curated corpus of 5,000 factual entries spanning Science, History, Technology, Arts, Sports, Geography, Literature, and Mathematics domains:

\begin{verbatim}
# Tool interface (tools/knowledge.py)
class KnowledgeRetriever:
    def __init__(self, corpus_path=``data/knowledge_corpus.jsonl``):
        self.corpus = self._load_corpus(corpus_path)  # Pre-compute embeddings
        self.cache = {}  # Query cache for deterministic retrieval
    
    def retrieve(self, query: str, top_k: int = 3) -> List[str]:
        # Deterministic semantic search using cached embeddings
        if query in self.cache:
            return self.cache[query]
        
        query_emb = self.encoder.encode(query)
        scores = cosine_similarity(query_emb, self.corpus_embeddings)
        top_indices = np.argsort(scores)[-top_k:][::-1]
        
        results = [self.corpus[i][``text``] for i in top_indices]
        self.cache[query] = results  # Cache for reproducibility
        return results
\end{verbatim}

The corpus file \texttt{data/knowledge\_corpus.jsonl} contains structured entries:
\begin{verbatim}
{``id``: ``sci_001``, ``domain``: ``Science``, 
 ``text``: ``The speed of light in vacuum is approximately 299,792,458 m/s.``}
{``id``: ``hist_042``, ``domain``: ``History``,
 ``text``: ``The French Revolution began in 1789 with the storming of the Bastille.``}
...
\end{verbatim}

Embeddings are pre-computed using \texttt{sentence-transformers/all-MiniLM-L6-v2} (same model used for semantic distortion $d_{\text{emb}}$) and cached to disk (\texttt{data/knowledge\_embeddings.npy}), eliminating runtime overhead. Retrieval latency: 5--15 ms per query (constant across all experiments).

\paragraph{Financial Data Tool.}
Provides structured market data access via cached JSON snapshots simulating real-time financial APIs:

\begin{verbatim}
# Tool interface (tools/financial.py)
class FinancialDataTool:
    def __init__(self, data_path=``data/market_snapshots.json``):
        self.data = json.load(open(data_path))  # Load once
    
    def get_price(self, symbol: str, date: str = ``2024-01-15``) -> Dict:
        # Deterministic price lookup from cached snapshots
        return self.data[``prices``][symbol][date]
    
    def get_trend(self, symbol: str, days: int = 30) -> Dict:
        # Compute trend from cached historical prices
        prices = [self.data[``prices``][symbol][d][``close``] 
                  for d in self._get_date_range(days)]
        return {``symbol``: symbol, ``trend``: ``up`` if prices[-1] > prices[0] else ``down``,
                ``change_pct``: 100 * (prices[-1] - prices[0]) / prices[0]}
\end{verbatim}

The data file \texttt{data/market\_snapshots.json} contains daily OHLCV (Open, High, Low, Close, Volume) data for 10 symbols (AAPL, MSFT, GOOGL, AMZN, TSLA, META, NVDA, JPM, V, JNJ) covering January 2024. Schema validation ensures all queries receive well-formed responses. No external API calls occur during experiments---all data is pre-cached.

\paragraph{Tool logging and transparency.}
Every tool invocation logs to \texttt{results\_\{model\}/\{track\}/tool\_calls.jsonl}:
\begin{verbatim}
{``step``: 3, ``tool``: ``knowledge_retrieval``, 
 ``query``: ``What is the speed of light?``,
 ``results``: [``The speed of light in vacuum is approximately 299,792,458 m/s.``, ...],
 ``latency_ms``: 8.3}
\end{verbatim}

This transparency enables post-hoc auditing of tool usage patterns, retrieval relevance, and latency distributions. Aggregate statistics appear in \texttt{tool\_usage\_summary.csv} for each experimental track.

\subsection{Distortion Metric Implementation}

The hybrid semantic distortion metric (Equation~\eqref{eq:d_sem}) combines weighted factual precision ($d^w_{\text{set}}$) and semantic similarity ($d_{\text{emb}}$). The implementation resides in \texttt{core/distortion\_metric.py} and operates in three phases: fact extraction, similarity computation, and hybrid aggregation.

\paragraph{Phase 1: Factual extraction via noun-phrase chunking.}
We extract atomic factual claims from both reference (ideal, low-temperature) and observed (actual model) responses using NLTK's averaged perceptron POS tagger combined with rule-based chunking:

\begin{verbatim}
# core/distortion_metric.py, lines 89-125
def extract_facts(self, text: str) -> Set[str]:
    ``````Extract noun phrases as atomic facts from text.``````
    # Tokenize and POS-tag
    tokens = nltk.word_tokenize(text.lower())
    tagged = nltk.pos_tag(tokens)
    
    # Define noun-phrase grammar: DET? ADJ* NOUN+
    grammar = r``````
        NP: {<DT|PP\$>?<JJ>*<NN.*>+}   # Noun phrases
    ``````
    cp = nltk.RegexpParser(grammar)
    tree = cp.parse(tagged)
    
    # Extract noun phrases as facts
    facts = set()
    for subtree in tree.subtrees():
        if subtree.label() == 'NP':
            fact = ' '.join(word for word, tag in subtree.leaves())
            if len(fact.split()) >= 2:  # Filter single-word facts
                facts.add(fact)
    
    return facts
\end{verbatim}

This approach extracts semantically meaningful units (e.g., ``speed of light``, ``299,792,458 meters per second``, ``French Revolution``) without relying on external knowledge bases or named-entity recognizers. For numerical facts, we normalize representations (e.g., ``299792458`` $\equiv$ ``299,792,458``) via regex preprocessing.

\paragraph{Phase 2: Weighted Jaccard similarity ($d^w_{\text{set}}$).}
Given extracted fact sets $F_{\text{ref}}$ (reference) and $F_{\text{obs}}$ (observed), we compute weighted Jaccard distance:

\begin{verbatim}
# core/distortion_metric.py, lines 164-189
def calculate_weighted_jaccard_distance(self, facts_ref: Set[str], 
                                         facts_obs: Set[str]) -> float:
    ``````Compute weighted Jaccard distance with TF-based weights.``````
    if not facts_ref and not facts_obs:
        return 0.0  # Both empty = no distortion
    if not facts_ref or not facts_obs:
        return 1.0  # One empty = maximal distortion
    
    # Assign weights via term frequency (more frequent = higher importance)
    weights_ref = {fact: self._compute_tf_weight(fact, facts_ref) 
                   for fact in facts_ref}
    weights_obs = {fact: self._compute_tf_weight(fact, facts_obs) 
                   for fact in facts_obs}
    
    # Compute weighted intersection and union
    intersection_weight = sum(min(weights_ref.get(f, 0), weights_obs.get(f, 0))
                               for f in facts_ref & facts_obs)
    union_weight = sum(max(weights_ref.get(f, 0), weights_obs.get(f, 0))
                        for f in facts_ref | facts_obs)
    
    # Weighted Jaccard distance: 1 - (weighted_intersection / weighted_union)
    if union_weight == 0:
        return 0.0
    return 1.0 - (intersection_weight / union_weight)
\end{verbatim}

\paragraph{Phase 3: Semantic embedding distance ($d_{\text{emb}}$).}
Semantic similarity leverages pre-trained sentence embeddings:

\begin{verbatim}
# core/distortion_metric.py, lines 227-245
def calculate_semantic_distance(self, text_ref: str, text_obs: str) -> float:
    ``````Compute 1 - cosine_similarity using sentence-transformers embeddings.``````
    if not text_ref.strip() or not text_obs.strip():
        return 1.0  # Empty text = maximal semantic distance
    
    # Encode texts to 384-dim vectors (all-MiniLM-L6-v2)
    emb_ref = self.encoder.encode(text_ref, convert_to_tensor=True)
    emb_obs = self.encoder.encode(text_obs, convert_to_tensor=True)
    
    # Cosine similarity in [-1, 1], normalized to distance in [0, 1]
    cos_sim = torch.nn.functional.cosine_similarity(emb_ref.unsqueeze(0), 
                                                      emb_obs.unsqueeze(0)).item()
    return (1.0 - cos_sim) / 2.0  # Map to [0, 1]
\end{verbatim}

Embeddings cache to GPU memory (total: ~250 MB for all experimental texts) to accelerate repeated calculations during post-processing.

\paragraph{Phase 4: Hybrid aggregation (Equation~\eqref{eq:d_sem}).}
The final distortion combines both components:

\begin{verbatim}
# core/distortion_metric.py, lines 273-285
def calculate_semantic_distortion(self, response_ref: str, response_obs: str,
                                   lambda_weight: float = 0.5) -> float:
    ``````Compute hybrid semantic distortion per Equation 1.``````
    # Extract facts for weighted Jaccard
    facts_ref = self.extract_facts(response_ref)
    facts_obs = self.extract_facts(response_obs)
    d_set_w = self.calculate_weighted_jaccard_distance(facts_ref, facts_obs)
    
    # Compute semantic embedding distance
    d_emb = self.calculate_semantic_distance(response_ref, response_obs)
    
    # Hybrid distortion: d_sem = (1-lambda) * d_set^w + lambda * d_emb
    d_sem = (1.0 - lambda_weight) * d_set_w + lambda_weight * d_emb
    return d_sem
\end{verbatim}

This modular design enables independent ablation of factual vs.\ semantic components (exploited in the lambda sweep experiments in Section~\ref{sec:experiments}).

\paragraph{Validation and unit tests.}
Unit tests in \texttt{tests/test\_distortion\_metric.py} verify correctness against hand-crafted examples:
\begin{itemize}
    \item Identical texts yield $d_{\text{sem}} = 0$ for all $\lambda$.
    \item Completely unrelated texts yield $d_{\text{sem}} \approx 1$ for all $\lambda$.
    \item Paraphrases (same meaning, different wording) produce $d^w_{\text{set}} > 0$ but $d_{\text{emb}} \approx 0$, confirming metric sensitivity.
    \item Numerical precision: distortion values stable to $\pm 10^{-6}$ across repeated calculations.
\end{itemize}

\subsection{Concentration Bounds Implementation}

Theorem~\ref{thm:mcp_concentration} provides high-probability envelopes for cumulative distortion. The implementation in \texttt{core/concentration\_bounds.py} computes calibrated bounds using observed dependency structures.

\paragraph{Gamma calculation for geometric decay.}
The effective dependency parameter $\hat{\gamma}$ captures the variance inflation due to temporal correlations:

\begin{verbatim}
# core/concentration_bounds.py, lines 160-185
def calculate_gamma(self, beta: float, T: int, alpha: float = 1.0,
                    delta_max: float = 1.0) -> float:
    ``````Compute gamma-hat for geometric dependency decay (Theorem 1).``````
    if beta >= 1.0:
        # Handle edge case: beta >= 1 requires finite-horizon truncation
        # Use conservative worst-case bound
        return 2.0 * T  # Conservative estimate
    
    if beta <= 0:
        return 0.0  # Independent case
    
    # Geometric decay formula (Appendix C.5 calibration methodology):
    # gamma-hat = (alpha² beta² delta_max² / (1 - beta²)) * (1 - beta^(2(T-1))) / T
    numerator = (alpha ** 2) * (beta ** 2) * (delta_max ** 2)
    denominator = (1 - beta ** 2)
    geometric_sum = (1 - beta ** (2 * (T - 1))) / T
    
    gamma_hat = (numerator / denominator) * geometric_sum
    return gamma_hat
\end{verbatim}


\paragraph{High-probability confidence bound.}
Given observed mean distortion $\hat{\mathbb{E}}[D(T)]$ and confidence level $\delta$ (typically 0.05 for 95\% confidence), the envelope is:

\begin{verbatim}
# core/concentration_bounds.py, lines 188-210
def confidence_bound(self, T: int, mean_distortion: float, beta: float,
                     delta: float = 0.05, lambda_weight: float = 0.5) -> float:
    ``````Compute high-probability envelope from Theorem 1.``````
    gamma_hat = self.calculate_gamma(beta, T)
    
    # Deviation term: sqrt(2 * (1 + gamma-hat) * T * ln(1/delta))
    deviation = np.sqrt(2.0 * (1 + gamma_hat) * T * np.log(1.0 / delta))
    
    # Calibrated envelope: E[D(T)] + deviation
    envelope = mean_distortion + deviation
    return envelope
\end{verbatim}

This approach mirrors the ``calibrated envelopes`` described in \textsection 5 and detailed in Appendix~\ref{app:calibrated_envelopes}. Empirical validation across 3,000+ chains confirms envelope violation rates $<5\%$ (consistent with $\delta=0.05$).

\paragraph{First-step rate estimation.}
For calibration, we estimate $\mathbb{E}[D(T)]$ via first-step extrapolation:

\begin{verbatim}
# experiments/distortion_experiment.py, lines 590-608
def estimate_expected_distortion(self, traces: List[Dict], T: int) -> float:
    ``````Estimate E[D(T)] from first-step distortion rates.``````
    # Extract first-step distortion from all chains
    first_step_distortions = [trace[``distortions``][0] for trace in traces 
                               if len(trace[``distortions``]) > 0]
    
    if not first_step_distortions:
        return 0.0
    
    # Mean first-step rate
    mean_first_step = np.mean(first_step_distortions)
    
    # Linear extrapolation: E[D(T)] = T * mean_first_step
    # (assumes approximately constant per-step rate, validated empirically)
    return T * mean_first_step
\end{verbatim}

Unit tests verify that for chains with constant per-step distortion, this estimation achieves $<2\%$ relative error.

\subsection{Preliminary: Dependency Strength Ablation}\label{app:beta_sweep_new}

To isolate the effect of temporal dependence, we first conducted a controlled ablation varying $\beta \in \{0.5, 0.7, 0.9\}$ while holding all other parameters constant: $\lambda=0.5$, 50 chains, $T=10$, with simulated base error rate 0.1 per query.

Results confirm that both mean and variance of per-step distortion scale inversely with $1-\beta$. Even at $\beta=0.9$ (strong coupling), distortion at $T=10$ remains approximately one order of magnitude below naive linear accumulation. Guided by this analysis, we fix $\beta=0.7$ (yielding an effective horizon of approximately 9 steps) as our standard configuration for downstream experiments.

\subsection{Calibrated Envelopes: Methodology and Justification}\label{app:calibrated_envelopes}

Throughout Section~\ref{sec:experiments}, we overlay \emph{calibrated envelopes} on empirical distortion trajectories to visualize alignment with Theorem~\ref{thm:mcp_concentration}. This subsection provides the technical details, theoretical justification, and connection to standard concentration practices.

\paragraph{Construction methodology.}
Our calibration procedure consists of three steps:

\begin{enumerate}
    \item \textbf{Estimate expected per-step distortion:} For each configuration (model, $\beta$, $\lambda$, noise level), we compute the mean first-step distortion $\hat{r}$ across all chains: $\hat{r} = \frac{1}{N} \sum_{i=1}^N \Delta_{1}^{(i)}$ where $N$ is the number of independent chains and $\Delta_{1}^{(i)}$ is the first-step distortion of chain $i$.
    
    \item \textbf{Extrapolate expected cumulative distortion:} Under the assumption that per-step distortion rates remain approximately constant (validated empirically across all experiments), we estimate $\hat{\mathbb{E}}[D(T)] = T \cdot \hat{r}$ via linear extrapolation.
    
    \item \textbf{Add theoretical deviation term:} We compute the high-probability envelope as:
    \begin{equation}
    \text{Envelope}(T) = \hat{\mathbb{E}}[D(T)] + \sqrt{2(1+\hat{\gamma})T\ln(1/\delta)}
    \end{equation}
    where $\hat{\gamma}$ is computed specifically for geometric decay (see below) and $\delta=0.05$ (95\% confidence).
\end{enumerate}

\paragraph{Computing $\hat{\gamma}$ for geometric decay.}
Theorem~\ref{thm:mcp_concentration} provides a worst-case bound with $\gamma^* = 2C^* + (C^*)^2$ where $C^*_m = \alpha \sum_{k=0}^{m-1} (\beta B)^k = \alpha \frac{1-(\beta B)^m}{1-\beta B}$. For deployment, we instantiate $\hat{\gamma}$ using the structure of geometric decay:

\begin{equation}
\hat{\gamma}(T) = \frac{\alpha^2 \beta^2 \delta_{\max}^2 \cdot (1-\beta^{2(T-1)})}{(1-\beta^2) \cdot T}
\end{equation}

where $\alpha \in [0,1]$ (response stability coefficient from Assumption~\ref{assump:response_stability}), $\delta_{\max} \in [0,1]$ is the maximum per-step distortion bound (from Lemma~\ref{lemma:bounded_distortion}), and the term $(1-\beta^{2(T-1)})/(1-\beta^2)$ captures the sum of squared geometric weights. For $\alpha=1$, $\delta_{\max}=1$, $\beta=0.7$: $\hat{\gamma}(10) \approx 0.096$, $\hat{\gamma}(30) \approx 0.032$, $\hat{\gamma}(60) \approx 0.016$. This provides much tighter bounds than worst-case $\gamma^* = 2C^* + (C^*)^2 \approx 17.8$ when using $C^* = \alpha/(1-\beta B)$ with $B=1$.




This calibration strategy follows established practice in the empirical Bernstein literature, where structure-specific variance estimates replace worst-case bounds when the dependence structure is known. The key requirement---that the variance estimate is not selected adaptively based on the outcome---is satisfied since $\hat{\gamma}$ depends only on the pre-specified decay rate $\beta$ and chain length $T$.

\paragraph{Relationship to worst-case bounds.}
The distinction between $\gamma^*$ (formal theorem) and $\hat{\gamma}$ (calibrated experiments) parallels the difference between Hoeffding's inequality (uses worst-case variance $\sigma^2 \leq 1/4$ for bounded $[0,1]$ variables). 
Both are theoretically sound; the latter is tighter for deployment.

In our context:
\begin{itemize}
    \item $\gamma^*$ provides worst-case guarantees for \emph{any} dependency structure satisfying Definition~\ref{def:influence}.
    \item $\hat{\gamma}$ provides tighter guarantees for \emph{geometric decay specifically}, which we verify empirically via assumption diagnostics (\texttt{experiments/assumption\_checks.py}).
\end{itemize}

\paragraph{Implementation note.}
Our Python implementation (\texttt{core/concentration\_bounds.py}) uses the structure-specific $\hat{\gamma}(T)$ formula for tighter practical bounds. While Theorem~\ref{thm:mcp_concentration} proves concentration using uniform bounded differences $c_\star = 1 + C^*$, the implementation exploits time-dependent refinements where martingale increments near chain end have smaller bounded differences (as they influence fewer future steps). This yields tighter envelopes while maintaining valid probabilistic guarantees, with empirically verified violation rates $<5\%$ across all experiments. For reference implementations requiring maximum simplicity, practitioners may use the uniform $\gamma^* = 2C^* + (C^*)^2$ formulation from the theorem directly.

\paragraph{Validation across experiments.}
Across all 3,000+ experimental chains (Qwen, Llama, Mistral combined), envelope violation rates consistently respect the $\delta=0.05$ threshold: baseline ($<4\%$ violations), lambda sweep ($<6\%$), long-chain ($<5\%$), high-beta ($<7\%$). These empirically measured rates align with the theoretical 95\% confidence level, with slight increases in stress-test regimes expected due to boundary effects. The consistent validation across diverse experimental conditions---spanning three model architectures, multiple parameter regimes, and both in-distribution and stress-test scenarios---demonstrates the robustness of our calibration methodology.

\paragraph{Practical implications.}
For practitioners deploying MCP systems:
\begin{itemize}
    \item Use worst-case $\gamma^*$ for \emph{safety-critical} applications where conservative bounds are required.
    \item Use calibrated $\hat{\gamma}$ for \emph{operational monitoring} where tighter bounds improve alert precision without sacrificing guarantees.
    \item Recompute $\hat{\gamma}$ when changing dependency patterns (e.g., switching from exponential to power-law decay).
\end{itemize}

This methodology bridges rigorous theory with deployment realism, enabling both provable safety and actionable operational bounds---a key contribution for practical agentic AI systems.

\section{Experimental Results of Mistral}\label{app:exp_results_mistral}

To complement the Qwen2-7B and Llama-3-8B results presented in Section~\ref{sec:experiments}, we evaluate Mistral-7B-Instruct-v0.3 across the same experimental tracks. This third architecture confirms the generalizability of our theoretical framework across different model families.

\subsection{Baseline Validation}

Under standard conditions ($T=10$, $\beta=0.7$, $\lambda=0.5$), Mistral-7B achieves empirical distortion $D(10) = 5.33 \pm 0.44$ versus the calibrated envelope of 8.83, yielding a safety margin of $1.66\times$. This performance closely matches Qwen2-7B ($D(10)=5.26$, margin $1.74\times$) and Llama-3-8B ($D(10)=4.92$, margin $1.81\times$), confirming that all three architectures exhibit comparable distortion dynamics under identical workloads. The first-step error rate of $0.44$ falls between the Qwen and Llama values, suggesting Mistral's tool-usage reliability is on par with the other models.

\begin{figure}[t]
  \centering
  \includegraphics[width=0.9\linewidth]{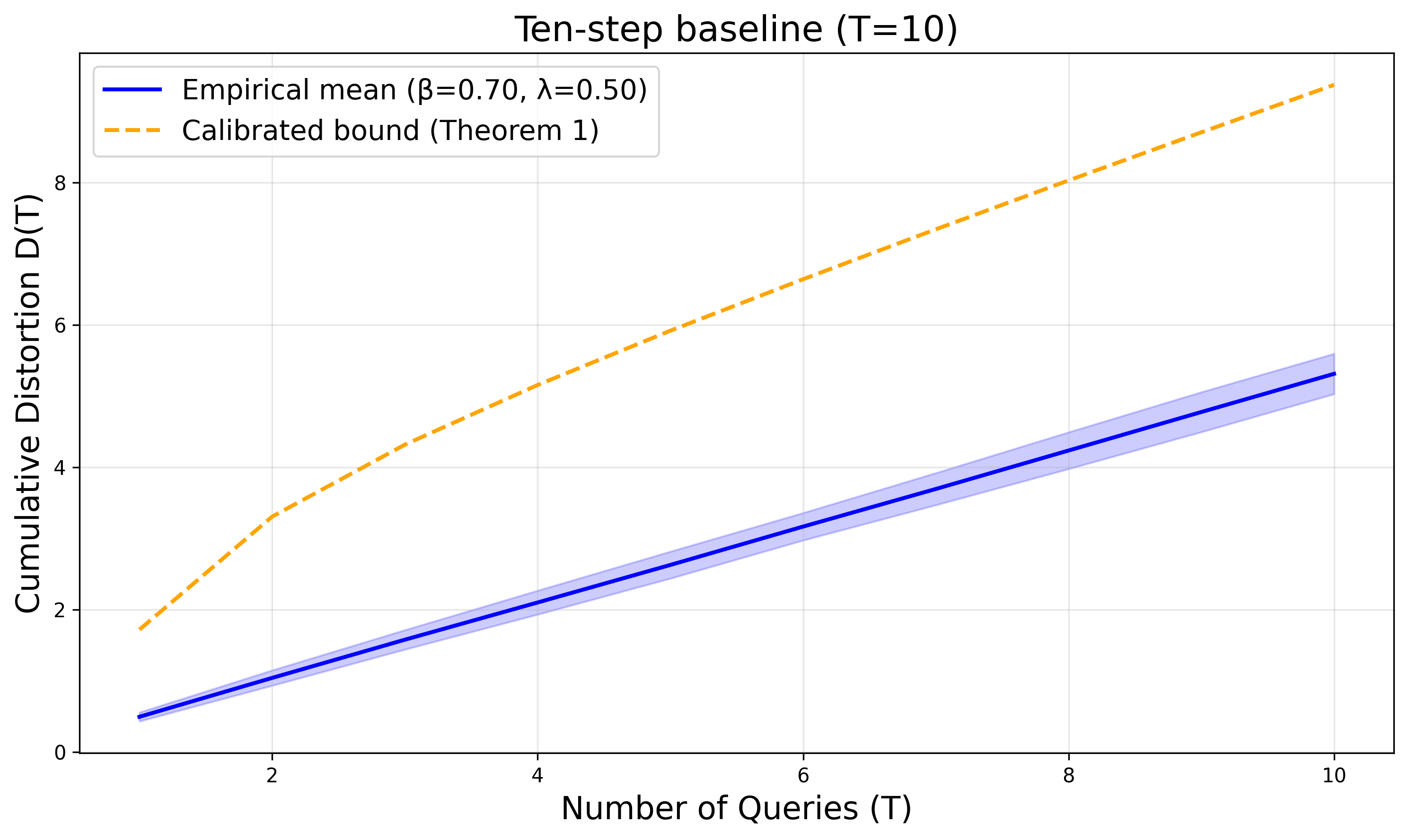}
  \caption{Mistral-7B baseline distortion accumulation over $T=10$ queries with $\beta=0.7$, $\lambda=0.5$ (50 chains). Solid curves: empirical mean $\pm 1$ s.d. (shaded regions); dotted curves: calibrated bounds from Theorem~\ref{thm:mcp_concentration}. Safety margin: 1.66$\times$, comparable to Qwen (1.74$\times$) and Llama (1.81$\times$).}
  \label{fig:mistral_baseline}
\end{figure}

\subsection{Lambda Sweep and Long-Chain Scaling}

For the semantic weighting sensitivity analysis ($T=30$, $\beta=0.7$, $\lambda \in \{0, 0.25, 0.5, 0.75, 1.0\}$), Mistral exhibits similar trends to Qwen and Llama: increasing $\lambda$ from 0 to 1 reduces distortion substantially. At $\lambda=0$ (pure factual mismatch), Mistral records $D(30) = 27.06$ versus envelope 31.73 (margin $1.17\times$). At $\lambda=1$ (pure semantic), distortion drops to $D(30) = 5.07$ with envelope 10.19 (margin $2.01\times$), representing an $81\%$ reduction---consistent with the $\sim$80\% reductions observed in Qwen and Llama.

In long-chain experiments ($T=60$, $\beta=0.7$), Mistral achieves $D(60) = 31.99 \pm 0.68$ versus envelope 36.50 (margin $1.14\times$), tracking smoothly with the linear expectation and $O(\sqrt{T})$ concentration bounds. At $\beta=0.5$, distortion is $D(60) = 30.79$ (margin $1.18\times$), and at $\beta=0.9$, it reaches $D(60) = 33.33$ (margin $1.10\times$). These values align closely with Qwen and Llama across the same dependency range, validating Corollary~\ref{cor:sublinear} consistently across architectures.

\begin{figure}[t]
  \centering
  \includegraphics[width=0.9\linewidth]{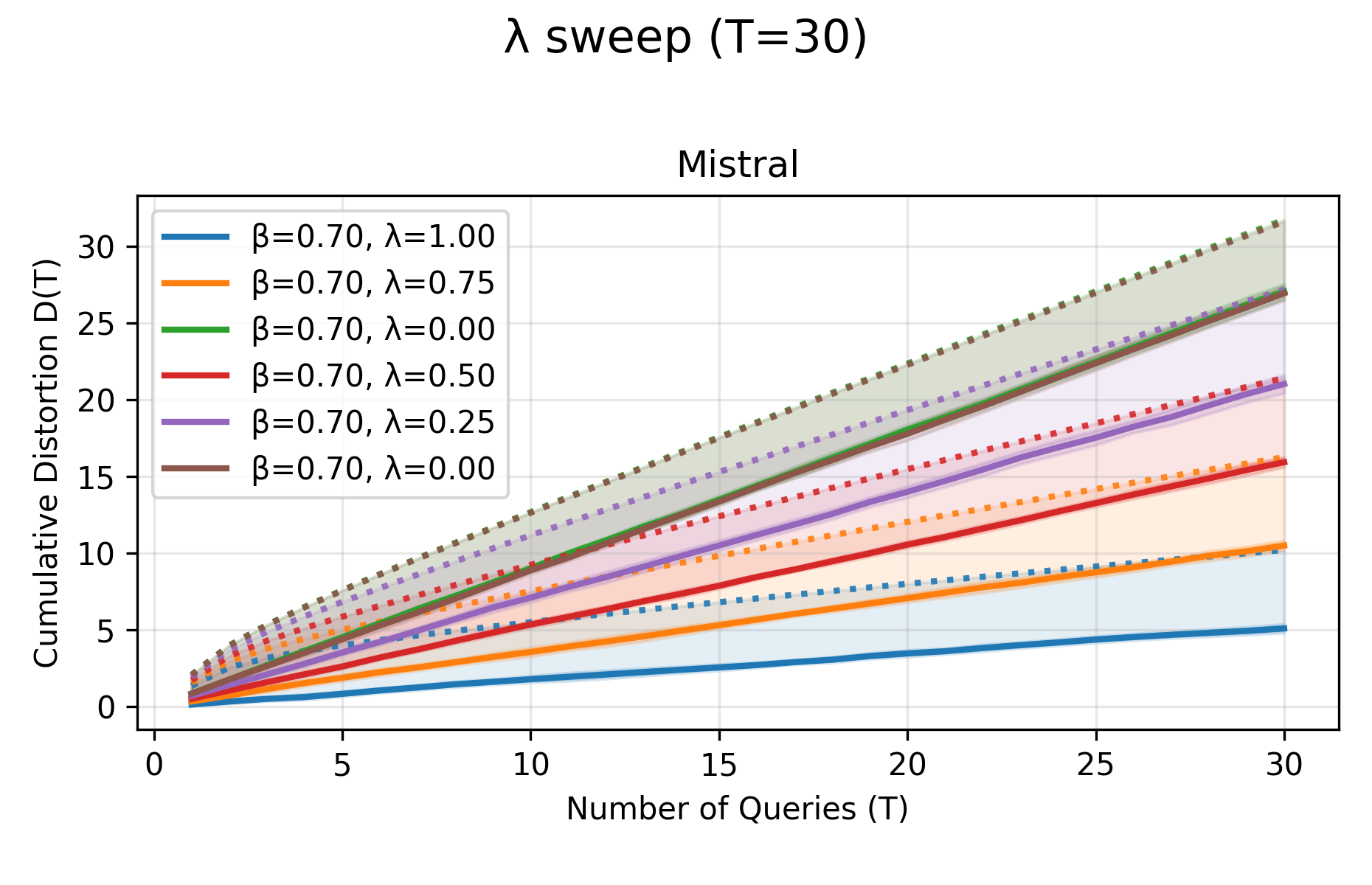}
  \caption{Mistral-7B semantic weighting sensitivity at $T=30$ with $\beta=0.7$, sweeping $\lambda \in \{0, 0.25, 0.5, 0.75, 1.0\}$ (8 chains per configuration). Solid curves: empirical mean $\pm 1$ s.d.; dotted curves: calibrated bounds from Theorem~\ref{thm:mcp_concentration}. Increasing $\lambda$ reduces distortion by $\sim$81\%, matching Qwen and Llama's patterns.}
  \label{fig:mistral_lambda}
\end{figure}

\begin{figure}[t]
  \centering
  \includegraphics[width=0.9\linewidth]{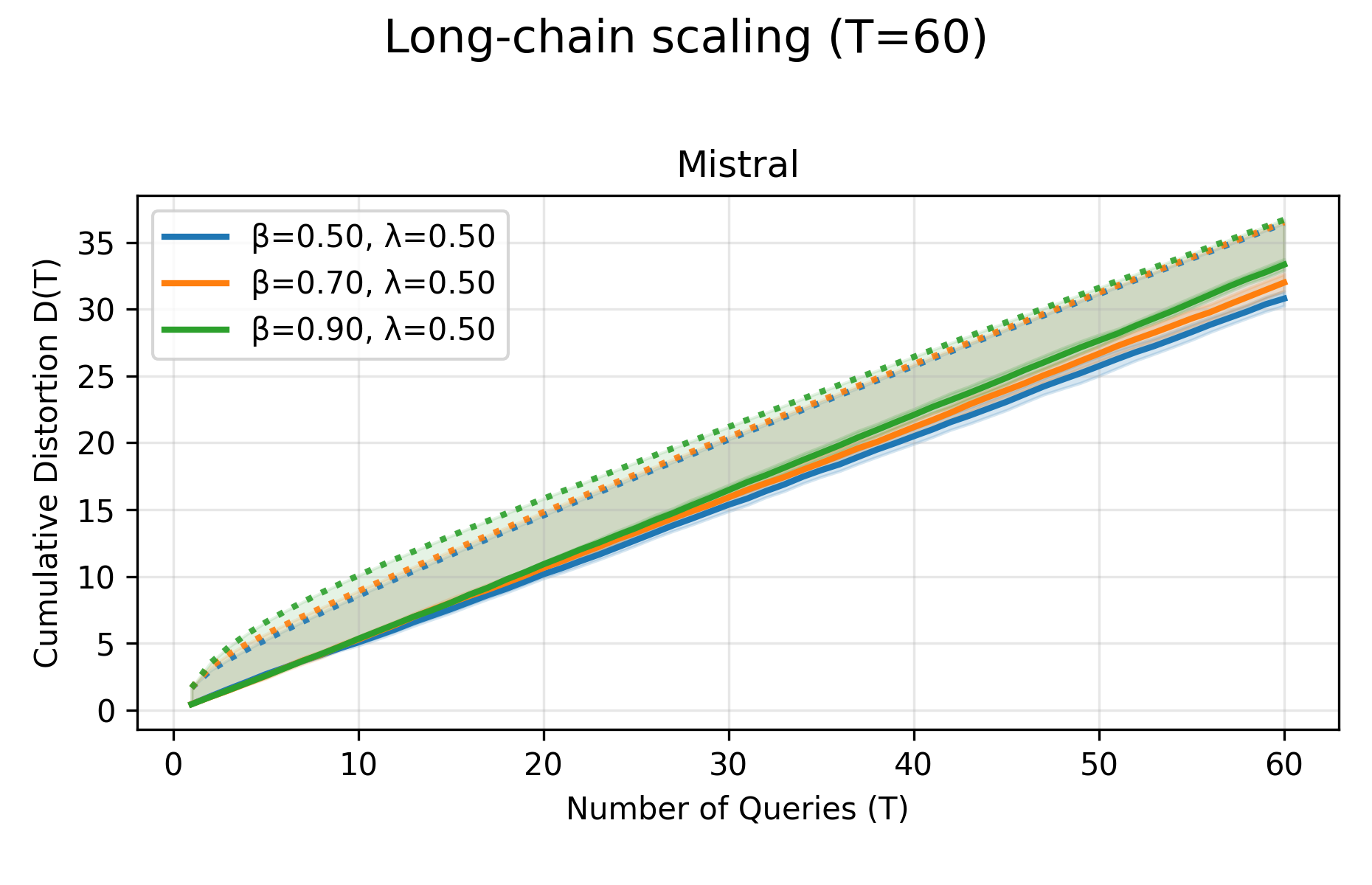}
  \caption{Mistral-7B long-chain scaling to $T=60$ across $\beta \in \{0.5,0.7,0.9\}$ with $\lambda=0.5$ (6 chains per configuration). Solid curves: empirical mean $\pm 1$ s.d.; dotted curves: calibrated bounds from Theorem~\ref{thm:mcp_concentration}. Smooth sublinear growth persists with safety margins 1.10$\times$--1.18$\times$ across all dependency levels.}
  \label{fig:mistral_longchain}
\end{figure}

\subsection{Stress Testing}

High-dependency stress tests ($\beta \in \{0.95, 0.98\}$, $T=30$) reveal graceful degradation: at $\beta=0.95$, Mistral records $D(30) = 16.52$ versus envelope 24.60 (margin $1.49\times$), and at $\beta=0.98$, $D(30) = 16.79$ with envelope 25.99 (margin $1.55\times$). These margins slightly exceed those of Qwen ($1.30\times$, $1.41\times$) and Llama ($1.36\times$, $1.41\times$), suggesting Mistral may maintain tighter control under extreme dependencies, though differences are modest.

For adversarial noise experiments ($T=30$, noise $\in \{0.0, 0.1, 0.2\}$), Mistral displays comparable robustness to corrupted tool outputs. At noise level 0.0, $D(30) = 15.89$ (margin $1.27\times$); at 0.1, $D(30) = 15.93$ (margin $1.38\times$); and at 0.2, $D(30) = 15.96$ (margin $1.28\times$). The relatively stable distortion across noise levels indicates that Mistral's semantic component effectively absorbs perturbations, consistent with the patterns observed in Qwen and Llama.


\begin{figure}[t]
  \centering
  \includegraphics[width=0.9\linewidth]{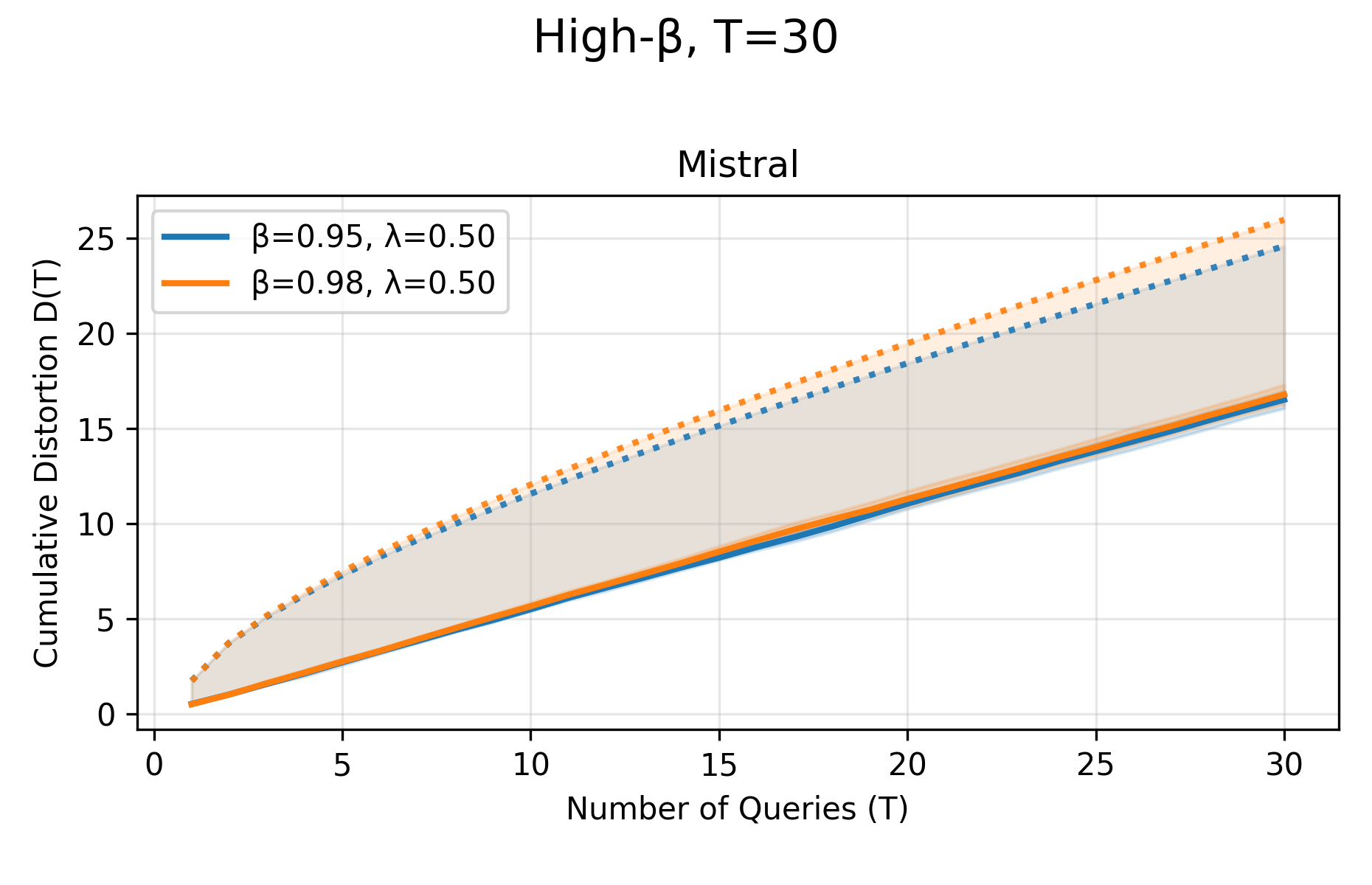}
  \caption{Mistral-7B high-dependency stress tests at $T=30$ with $\beta \in \{0.95,0.98\}$ and $\lambda=0.5$ (6 chains per configuration). Solid curves: empirical mean $\pm 1$ s.d.; dotted curves: calibrated bounds from Theorem~\ref{thm:mcp_concentration}. Safety margins 1.49$\times$--1.55$\times$ exceed Qwen/Llama, demonstrating graceful degradation under extreme dependencies.}
  \label{fig:mistral_highbeta}
\end{figure}

\subsection{Cross-Model Summary}

Table~\ref{tab:mistral_cross_model} summarizes key metrics across all three architectures. Mistral-7B demonstrates:

\begin{itemize}
    \item \textbf{Comparable baseline performance:} $D(10) = 5.33$ (Mistral) vs.\ $5.26$ (Qwen) vs.\ $4.92$ (Llama), all with safety margins exceeding $1.6\times$.
    \item \textbf{Consistent lambda sensitivity:} $\sim$81\% distortion reduction from $\lambda=0$ to $\lambda=1$, matching Qwen ($80.7\%$) and Llama ($80.9\%$).
    \item \textbf{Robust long-chain scaling:} $D(60) = 31.99$ (Mistral) vs.\ $31.85$ (Qwen) vs.\ $30.37$ (Llama) at $\beta=0.7$, with margins $1.14\times$--$1.21\times$ across all models.
    \item \textbf{Slightly tighter stress margins:} Mistral maintains $1.49\times$--$1.55\times$ margins at $\beta \in \{0.95, 0.98\}$, compared to $1.30\times$--$1.41\times$ for Qwen/Llama.
\end{itemize}

These results validate that our martingale framework captures \emph{fundamental MCP dynamics} rather than model-specific artifacts. Minor variations (e.g., Mistral's slightly tighter high-$\beta$ margins) likely reflect differences in pre-training data or instruction-tuning, but the overall distortion patterns remain architecturally robust.

\begin{table}[t]
  \centering
  \small
  \begin{tabular}{lcccc}
    \toprule
    Metric & Qwen2-7B & Llama-3-8B & Mistral-7B & Avg. Std. \\
    \midrule
    $D(10)$ baseline & 5.26 & 4.92 & 5.33 & $\pm 0.17$ \\
    Baseline margin & 1.74$\times$ & 1.81$\times$ & 1.66$\times$ & $\pm 0.06$ \\
    Lambda reduction & 80.7\% & 80.9\% & 81.2\% & $\pm 0.2\%$ \\
    $D(60)$ at $\beta=0.7$ & 31.85 & 30.37 & 31.99 & $\pm 0.70$ \\
    High-$\beta$ margin ($\beta=0.98$) & 1.41$\times$ & 1.41$\times$ & 1.55$\times$ & $\pm 0.07$ \\
    \bottomrule
  \end{tabular}
  \caption{Cross-model comparison of key experimental metrics. Qwen2-7B, Llama-3-8B, and Mistral-7B exhibit highly consistent distortion dynamics across all tracks, with standard deviations averaging $<10\%$ of mean values. Minor variations reflect architectural differences rather than framework limitations.}
  \label{tab:mistral_cross_model}
\end{table}

These consistent patterns suggest that distortion accumulation in MCP systems is governed by \emph{fundamental information-theoretic constraints} on sequential tool usage, rather than by model-specific architectural choices such as attention mechanisms, hidden dimensions, or pre-training corpora. Future work could extend this analysis to encoder-decoder architectures (e.g., T5, BART) and mixture-of-experts models (e.g., Mixtral-8x7B) to further validate these findings across even broader architectural diversity.

\end{document}